%% file: root.tex
\documentclass[letterpaper, 10 pt, conference]{ieeeconf}  

\IEEEoverridecommandlockouts                              
\usepackage{epsfig} 
\usepackage{mathptmx} 
\usepackage{times} 
\usepackage{amsmath} 
\usepackage{amssymb}  
\usepackage{algorithm}
\usepackage{algpseudocode}
\usepackage{bm,bbm}
\usepackage{xcolor}
\usepackage{cite}
\usepackage[colorlinks=true, citecolor=blue]{hyperref}
\usepackage{cleveref}
\usepackage{kotex}
\usepackage{dsfont}
\usepackage{subcaption}

\title{\LARGE \bf
Near-Optimal Primal-Dual Algorithm for Learning Linear Mixture CMDPs with Adversarial Rewards
}

\author{
\parbox{3in}{\centering
Kihyun Yu, Seoungbin Bae\\
Department of Industrial and Systems Engineering, KAIST\\
{\tt\small \{khyu99, sbbae31\}@kaist.ac.kr}
}
\hspace*{0.5in}
\parbox{3in}{\centering
Dabeen Lee*\\
Department of Mathematical Sciences, \\ Seoul National University\\
{\tt\small dabeenl@snu.ac.kr}
}
\thanks{*Additional affiliations: Research Institute of Mathematics, Seoul National University; Interdisciplinary Program in Artificial Intelligence, Seoul National University; Korea Institute for Advanced Study.}
}

\input{math}

\begin{document}

\maketitle
\thispagestyle{empty}
\pagestyle{empty}



\input{arXiv}

\section{ACKNOWLEDGMENTS}

This work was supported by the National Research Foundation of Korea (NRF) grant (No. RS-2024-00350703) and the Institute of Information \& communications Technology Planning \& evaluation (IITP) grants (No. IITP-2026-RS-2024-00437268) and (No. RS-2021-II211343, Artificial Intelligence Graduate School Program (Seoul National University)) funded by the Korea government (MSIT).

\bibliographystyle{IEEEtran}
\bibliography{ref}
ChatGPT was used to improve the syntax and grammar of several paragraphs in the manuscript. Additionally, ChatGPT assisted with the implementation of the numerical experiments.

\onecolumn
\newpage
\appendices
\input{appendix}







\end{document}

%% file: math.tex
\newtheorem{theorem}{Theorem}
\newtheorem{lemma}{Lemma}

\newtheorem{proposition}{Proposition}

\newtheorem{definition}{Definition}
\newtheorem{assumption}{Assumption}
\newtheorem{remark}{Remark}

\newcommand{\calA}{{\mathcal{A}}}

\newcommand{\calC}{{\mathcal{C}}}

\newcommand{\calE}{{\mathcal{E}}}
\newcommand{\calF}{{\mathcal{F}}}
\newcommand{\calG}{{\mathcal{G}}}

\newcommand{\calM}{{\mathcal{M}}}

\newcommand{\calO}{{\mathcal{O}}}

\newcommand{\calS}{{\mathcal{S}}}

\newcommand{\bbE}{\mathbb{E}}

\newcommand{\bbR}{\mathbb{R}}

\newcommand{\bbV}{\mathbb{V}}

\newcommand{\bfb}{\mathbf{b}}

\newcommand{\bfx}{\mathbf{x}}

\newcommand\Id{\text{I}}

\newcommand{\piunif}{\pi_{\mathrm{unif}}}

\newcommand{\interior}{\operatorname{int}}
\newcommand{\bigO}{\widetilde{{O}}}
\newcommand{\AlgName}{PD-POWERS }

\DeclareMathOperator*{\argmax}{arg\,max}
\DeclareMathOperator*{\argmin}{arg\,min}

\newcommand\Regret{\mathrm{Reg}}
\newcommand\Violation{\mathrm{Vio}}

%% file: arXiv.tex

\begin{abstract}
We study safe reinforcement learning in finite-horizon linear mixture constrained Markov decision processes (CMDPs) with adversarial rewards under full-information feedback and an unknown transition kernel. We propose a primal-dual policy optimization algorithm that achieves regret and constraint violation bounds of $\widetilde{O}(\sqrt{d^2 H^3 K})$, where $d$ is the feature dimension, $H$ is the horizon, and $K$ is the number of episodes. To the best of our knowledge, this is the first provably efficient algorithm for linear mixture CMDPs with adversarial rewards. In particular, our regret bound is near-optimal, matching the known minimax lower bound up to logarithmic factors. The key idea is to introduce a regularized dual update that enables a drift-based analysis. This step is essential, as strong duality-based analysis cannot be directly applied when reward functions change across episodes. In addition, we extend weighted ridge regression–based parameter estimation to the constrained setting, allowing us to construct tighter confidence intervals that are crucial for deriving the near-optimal regret bound.
\end{abstract}

\section{Introduction}
In this paper, we study online linear mixture constrained Markov decision processes (CMDPs) with adversarial rewards. A CMDP can be described by the following standard formulation:\footnote{We note that \eqref{eq:CMDP simple} represents the standard CMDP formulation; the formal definition of linear mixture CMDPs with adversarial rewards will be introduced in \Cref{sec:problem formulation}.}
\begin{equation}\label{eq:CMDP simple}
    \max_{\pi} \ V_{1}^{r,\pi}(s_1)
    \quad \textnormal{subject to} \quad
    V_1^{g,\pi}(s_1) \geq b,
\end{equation}
where $\pi$ denotes a policy, $V_1^{r,\pi}(s_1) = \bbE_\pi[\sum_{h=1}^H r_h(s_h,a_h)\mid s_1]$, $V_1^{g,\pi}(s_1) = \bbE_\pi[\sum_{h=1}^H g_h(s_h,a_h)\mid s_1]$, and $b$ denotes the constraint threshold.

Recently, a large number of algorithms for learning CMDPs have been proposed~\cite{efroni2020exploration}. However, many existing approaches are limited to settings with tabular state spaces~\cite{liu2021learning}, and their theoretical guarantees do not readily extend to large state spaces. To address scalability, one common direction in RL is to consider linear function approximation, where transition kernels admit a linear representation. In the CMDP literature, \cite{ding2021provably,ghosh2022provably, ghosh2024towards, kitamura2025provably} proposed algorithms in this setting; however, these works considered fixed rewards and constraints and thus could not capture non-stationary environments. To address this limitation, more recently, \cite{ding2023provably, yu2026primaldual} proposed algorithms for non-stationary environments; however, \cite{ding2023provably} required the variation budget—a quantity that characterizes the change in a CMDP—to be known and bounded, and \cite{yu2026primaldual} only attained suboptimal regret and violation bounds. In summary, the existing algorithms for safe RL with linear function approximation either suffer from suboptimal regret guarantees or rely on restrictive assumptions, such as stationarity or a bounded variation budget.

Motivated by these limitations, we design an algorithm for linear mixture CMDPs with adversarial rewards that achieves near-optimal regret bounds. In this setting, the linear mixture structure enables function approximation, while adversarial rewards capture non-stationary environments. 

This goal is particularly challenging. Specifically, when 
rewards change arbitrarily, strong duality-based analysis---a key tool in previous works~\cite{ding2021provably, ghosh2022provably}---cannot be directly applied. 
Moreover, achieving the optimal sample complexity requires constructing tighter confidence intervals. Despite these challenges, our main contributions are as follows.
\begin{itemize}
    \item We propose a provably efficient primal-dual policy optimization algorithm for finite-horizon linear mixture CMDPs with adversarial reward functions under full-information feedback, a fixed constraint, and an unknown transition kernel.
    
    \item Our algorithm achieves regret and constraint violation bounds of $\bigO(\sqrt{d^2 H^3 K})$ {in the high-dimensional regime}\footnote{{Here, the high-dimensional regime refers to $d=\widetilde\Omega(H^{7/2}/\gamma^2)$, where $\gamma$ is the Slater constant. A high-dimensional condition also appears in the unconstrained adversarial linear mixture MDP result of \cite{he2022near}; since we consider the constrained setting, our condition additionally depends on $\gamma$.}}, where $d$ is the feature dimension, $H$ is the horizon, and $K$ is the number of episodes. We emphasize that our regret bound is near-optimal in the sense that it matches the known minimax regret lower bound for unconstrained linear mixture MDPs up to logarithmic factors~\cite{he2022near}.
    
    \item The key idea of our algorithm is to combine primal-dual policy optimization with (i) a regularized dual update that enables a 
    drift-based analysis, which does not rely on strong duality, and (ii) weighted ridge regression–based parameter estimation to obtain tighter confidence intervals, which are crucial for attaining near-optimal regret bounds.
\end{itemize}

In \Cref{tab:main}, we compare algorithms for safe RL with linear function approximation. \cite{ghosh2023achieving, liu2025sample, wei2026nearoptimal} consider infinite-horizon linear CMDPs, and \cite{tian2024confident} studies $q^{\pi}$-realizable CMDPs. Although these works fall under safe RL with linear function approximation, their settings differ significantly from ours.
\begin{table}[H]
    \caption{Comparison of safe RL with linear function approximation. The Reg. \& Vio. column reports results in the high-dimensional regime.}
    \centering
    \begin{tabular}{|c|ccc|}
        \hline
        Algorithm & MDP Setting & Reward & Reg. \& Vio. \\
        \hline
        \cite{ding2021provably} & Linear Mixture  & Fixed & $\bigO(\sqrt{d^2 H^4 K})$ \\
        \cite{ghosh2022provably} & Linear & Fixed & $\bigO(\sqrt{d^3 H^4 K})$ \\
        \cite{yu2026primaldual} & Linear  & Adversarial & $\bigO(\textnormal{poly}(d,H)K^{3/4})$ \\
        \hline
        Ours & Linear Mixture & Adversarial & $\bigO(\sqrt{d^2 H^3 K})$ \\
        \hline
    \end{tabular}
    \label{tab:main}
\end{table}

\subsection{Additional Related Works}
We provide additional prior works on learning CMDPs. \cite{wei2022triple,bura2022dope,yu2025improved,muller2024truly,zhu2025optimistic,stradi2025optimal,liu2025nearoptimal} propose online algorithms for tabular CMDPs. \cite{qiu2020upper,stradi2024online,stradi2025learning,stradi2025policy} are particularly relevant to our work as they consider adversarial rewards; however, their results are restricted to the tabular setting. Beyond the tabular case, several works study CMDPs with function approximation. In particular, \cite{amani2021safe,shi2023near,wei2024safe,roknilamouki2025provably} focus on instantaneous constraints, where the constraint must be satisfied at each step rather than in expectation over the entire trajectory.

\section{Problem Formulation}\label{sec:problem formulation}
\noindent \textbf{Notation.} 
For a positive integer $n$, let $[n] = \{1,\ldots,n\}$. For $x\in\bbR^d$, let $\|x\|_A = \sqrt{x^\top A x}$ for some positive definite matrix $A \in \bbR^{d \times d}$. For $a,b \in \bbR$ such that $a \leq b$, let $[a,b] = \{x \in \bbR: a\leq x\leq b\}$, let $[\cdot]_+ = \max\{\cdot,0\}$, and let $[\cdot]_{[a,b]} = \max\{\min\{\cdot,b\},a\}$. Let $I \in \bbR^{d\times d}$ denote the identity matrix, let $\bm{0}\in \bbR^d$ denote the all-0 vector, and let $\langle\cdot,\cdot\rangle$ denote the inner product. Let $\Delta( A )$ denote the probability simplex over set $ A $.

\noindent \textbf{Finite-Horizon CMDP.} 
We consider a finite-horizon CMDP with adversarial rewards and a fixed constraint function. A finite-horizon CMDP is defined by $\calM = ( S ,  A , H, P, r, g, b, s_1)$, where $ S $ is the finite state space, $ A $ is the finite action space, and $H$ is the horizon. The collection of unknown transition kernels is denoted by $P = \{P_h\}_{h\in [H]}$, where $P_h(s'|s,a)$ is the probability of transitioning from state $s$ to state $s'$ by taking action $a$ at step $h$. The collections of reward and constraint functions are denoted by $r = \{r_h^k\}_{h\in [H], k\in [K]}$ and $g = \{g_h\}_{h\in [H]}$, where $r_h^k, g_h :  S  \times  A  \to [0,1]$ for each $h,k$. The constraint threshold is denoted by $b \in [0,H]$, and $s_1$ is the fixed initial state.

We assume adversarial deterministic rewards with full-information feedback, while the deterministic constraint function is fixed and known.\footnote{{While we assume a single fixed and deterministic constraint function for simplicity, our framework can be extended to multiple unknown stochastic constraint functions with linear structure under bandit feedback.}}. In particular, at the beginning of each episode $k$, an adversary chooses an arbitrary reward function $r^k = \{r_h^k\}_{h\in [H]}$. The full information of $r^k$ is revealed at the end of episode $k$. On the other hand, the constraint function $g = \{g_h\}_{h\in [H]}$ is assumed to be fixed across all $k \in [K]$ and known. 

The agent interacts with the environment as follows. At the beginning of episode $k$, the agent selects a policy $\{\pi_h^k\}_{h\in [H]}$, where $\pi_h^k(a|s)$ denotes the probability of taking action $a$ at state $s$ and step $h$. During episode $k$, for each step $h$, the agent observes the current state $s_h^k$ and selects an action $a_h^k \sim \pi_h^k(\cdot|s_h^k)$. The next state is then sampled as $s_{h+1}^k \sim P_h(\cdot|s_h^k,a_h^k)$. At the end of the episode, $r^k$ is revealed to the agent.

Given $P$ and $\pi$, we define the value function $V_{h}^{\ell,\pi}(s) = \mathbb{E}_{\pi}\![\sum_{j=h}^H \ell_j(s_j,a_j) |s_h=s],$
where $\ell = \{\ell_h\}_{h\in [H]}$ is any function such that $\ell_h: S \times A  \to [0,1]$, and $\mathbb{E}_{\pi}[\cdot]$ denotes the expectation over the trajectory $(s_h,a_h,\ldots,s_H,a_H)$ induced by $P$ and $\pi$. Likewise, we define the $Q$-function as $Q_{h}^{\ell,\pi}(s,a) = \mathbb{E}_{\pi}\![\sum_{j=h}^H \ell_j(s_j,a_j) | s_h=s, a_h=a].$

The goal of the agent is to learn an optimal policy $\pi^*\in \Pi$ for the following optimization problem, where for the set of policies $\Pi = \{\{\pi_h\}_{h\in [H]} : \pi_h: S  \to \Delta( A )\}$:
\begin{equation}\label{eq:CMDP}
    \max_{\pi \in \Pi} \ \sum_{k=1}^K V_{1}^{r^k,\pi}(s_1)
    \quad \textnormal{subject to} \quad
    V_1^{g,\pi}(s_1) \geq b.
\end{equation} 

The performance metrics are defined as follows. Given a sequence of policies $\{\pi^k\}_{k \in [K]}$, we define the regret as $\Regret(K) = \sum_{k=1}^K \big( V_1^{r^k,\pi^*}(s_1) - V_1^{r^k,\pi^k}(s_1) \big)$, and the constraint violation as $\Violation(K) = \left[\sum_{k=1}^K \big( b - V_{1}^{g,\pi^k}(s_1) \big)\right]_+$.

We introduce additional notations as follows. For any $V: S \to \bbR$, $P_hV(s,a) = \sum_{s'}P_h(s'|s,a)V(s')$ and $\bbV_h V(s,a) = P_hV^2(s,a) - (P_hV(s,a))^2$. Moreover, we present the Slater assumption, as in \cite{yu2026primaldual}.

\begin{assumption}[Slater condition]\label{assum:slater}
    There exists a Slater policy $\bar\pi$ such that $V_{1}^{g,\bar\pi}(s_1) \geq b+\gamma$, where $\gamma > 0$ is the Slater constant. Note that $\bar\pi$ and $\gamma$ are unknown to the agent.
\end{assumption}

\noindent \textbf{Linear Mixture CMDP.} Finally, we introduce the definition of linear mixture CMDPs, adapted from \cite{jia2020model, ayoub2020model, zhou2021nearly}, which we assume throughout the paper. 
\begin{definition}[Linear Mixture CMDP]\label{def:linear mixture}
    We say that $\calM$ is an inhomogeneous, episodic $B$-bounded linear mixture CMDP if, for each $h\in [H]$, there exist a known feature mapping $\phi: S \times A \times S  \to \bbR^d$ and an unknown parameter $\theta_h^* \in \bbR^d$ such that $P_h(s'|s,a) = \langle \phi(s'|s,a), \theta_h^* \rangle$ for any $(s,a,s') \in  S  \times  A  \times  S $. Moreover, we assume that $\|\theta_h^*\|_2 \leq B$ and $\|\phi_V(s,a)\|_2\leq 1$ for all $(s,a,h) \in  S  \times  A  \times[H]$ and any $V: S  \to [0,1]$, where $\phi_V(s,a) = \sum_{s'\in S } \phi(s'|s,a)V(s')$.
\end{definition}

\section{Proposed Algorithm}\label{sec:algorithm}
In this section, we present our algorithm, called \underline{P}rimal-\underline{D}ual \underline{P}olicy \underline{O}ptimization \underline{W}ith B\underline{ER}nstein bonu\underline{S} (PD-POWERS, \Cref{alg:main}), tailored for finite-horizon linear mixture CMDPs with adversarial rewards and a fixed constraint function. Intuitively, PD-POWERS can be viewed as a primal-dual variant of POWERS, proposed in \cite{he2022near} for the unconstrained setting.

\begin{algorithm}[t]
\caption{\AlgName}
\label{alg:main}
\textbf{Require: } regularization parameter $\lambda$; step sizes $\alpha,\eta$; mixing parameter $\theta$; the constraint function $\{g_h(\cdot,\cdot)\}_{h\in [H]}$\\
\textbf{Initialize: } $\forall (h,\ell)\in[H] \times \{r,g\}$, $\widehat\Sigma_{1,h}^\ell, \widetilde\Sigma_{1,h}^\ell \leftarrow \lambda\Id$; $\widehat b_{1,h}^\ell, \widetilde b_{1,h}^\ell \leftarrow \bm 0$; $\widehat \theta_{1,h}^\ell, \widetilde \theta_{1,h}^\ell \leftarrow \bm 0$; $V_{1,H+1}^r(\cdot), V_{1,H+1}^g(\cdot)$; $\pi_h^1 \leftarrow \piunif$; $Y_1 \leftarrow 0$;
\begin{algorithmic}[1]
    \For{$k=1,\ldots, K$}
        \If{$k>1$}
            \State Update $\{\pi_h^{k}\}_{h\in [H]}$ as in \eqref{eq:policy optimization}\label{line:policy optimization}
            \State Update $Y_{k}$ as in \eqref{eq:dual update}\label{line:dual update}
        \EndIf
        \For{$h=1,\ldots,H$}
            \State Take action $a_h^k \sim \pi_h^k(\cdot\mid s_h^k)$
            \State Receive $s_{h+1}^k \sim  P_h(\cdot \mid s_h^k, a_h^k)$
        \EndFor
        \State Receive $\{r_h^k(\cdot,\cdot)\}_{h\in [H]}$
        \For{$h=H,\ldots, 1$}\label{line:param start}
             \For{$\ell=r,g$}
                \State Compute $Q_{k,h}^\ell(\cdot,\cdot)$ as in \eqref{eq:Q}
                \State $V_{k,h}^\ell(\cdot)\leftarrow \sum_{a\in A } \pi_h^k(a\mid \cdot) Q_{k,h}^\ell(\cdot,a)$
                \State Set $\widehat\theta_{k+1,h}^\ell, \widehat\Sigma_{k+1,h}^\ell, \widehat b_{k+1,h}^\ell$   as in \eqref{eq:theta hat}
                \State Set $\widetilde \theta_{k+1,h}^\ell, \widetilde \Sigma_{k+1,h}^\ell, \widetilde b_{k+1,h}^\ell$  as in  \eqref{eq:theta tilde Sigma tilde b tilde}
                \State Set $\bar\bbV_hV_{k,h+1}^\ell$  as in  \eqref{eq:bar bbV}
                \State Set $E_{k,h}^\ell, \bar\sigma_{k,h}^\ell$  as in \eqref{eq:E}, \eqref{eq: bar sigma}
            \EndFor \label{line:param end}
        \EndFor
    \EndFor
\end{algorithmic}
\end{algorithm}

\subsection{Challenges in Algorithm Design}

To extend an algorithm for the unconstrained problem to the constrained setting, a standard approach is to construct a primal-dual variant. In this framework, the primal variable—maximizing a Lagrangian objective—and the dual variable—balancing reward maximization and constraint satisfaction—are updated alternately. 

The main challenge lies in the design of the dual update, as previously proposed dual updates do not directly apply to our setting with adversarial rewards. In the fixed-reward setup, \cite{ding2021provably, ghosh2022provably} adopt a dual update that naturally leads to a strong duality-based analysis. However, such an analysis becomes non-trivial when the rewards vary across episodes.

To address this issue, our key idea is to introduce a regularized dual update, inspired by \cite{yu2026primaldual}. This update enables a drift-based analysis that does not rely on strong duality. In particular, the regularization term induces a negative drift in the dual variable, which plays a crucial role in controlling its growth. This justifies the suitability of our dual update in the adversarial-reward setting. For more details, we refer the reader to \ref{subsec:description} and \ref{subsec:dual}.

\subsection{Description of PD-POWERS}\label{subsec:description}

We now describe PD-POWERS. We first introduce the policy optimization step, which combines entropy regularization with policy perturbation. We then present the regularized dual update and discuss its role in controlling the dual variable. Finally, we present the parameter estimation procedure based on weighted ridge regression.

\noindent\textbf{Policy Optimization} (Line \ref{line:policy optimization}).
Our policy optimization step proceeds in the following two steps:
\begin{align}\label{eq:policy optimization}
\begin{aligned}
    \widetilde \pi_h^{k-1}(\cdot | s) &\leftarrow (1-\theta)\pi_h^{k-1}(\cdot|s) + \theta \piunif(\cdot|s) \\
    \pi_h^{k}(\cdot | s) &\propto \widetilde\pi_h^{k-1}(\cdot|s)\exp(\alpha (Q_{{k-1},h}^r(s,\cdot) + Y_{k-1} Q_{{k-1},h}^g(s,\cdot)))
\end{aligned}
\end{align}
In the first step of \eqref{eq:policy optimization}, we define $\widetilde \pi_h^{k-1}$, which is a perturbed version of $\pi_h^{k-1}$. This step ensures that $\widetilde\pi_h^{k-1}(a|s) \geq \theta/|A|$ for all $a \in A$, and hence keeps $\widetilde\pi_h^{k-1}$ away from the boundary of the probability simplex. This is essential for controlling the growth of the dual variable. We refer the reader to \Cref{sec:analysis} for further technical motivation.

Moreover, a key difference from \cite{yu2026primaldual} is that we perturb the policy at every episode. In \cite{yu2026primaldual}, the perturbation is applied every $K^{3/4}$ episodes to handle the trade-off between the covering number\footnote{Note that \cite{yu2026primaldual} considers linear CMDPs, while our setting is linear mixture CMDPs. In their setting, the covering number of the value function class must be controlled to ensure uniform convergence over all possible value function estimates induced by their algorithm.} and the dual variable, which leads to suboptimal dependence on $K$ in the regret. In contrast, since our setting does not require such covering number arguments, it allows more frequent perturbations and yields optimal dependence on $K$.

In the second step of \eqref{eq:policy optimization}, we perform policy optimization following \cite{cai2020provably}. Equivalently, it can be rewritten as the following online mirror descent (OMD) step over the policy space: $\pi_h^{k}(\cdot|s) \in \argmax_{\pi\in\Pi} \langle \pi(\cdot|s), Q_{{k-1},h}^r(s,\cdot) + Y_{k-1} Q_{{k-1},h}^g(s,\cdot) \rangle - \frac{1}{\alpha}D(\pi(\cdot|s)||\widetilde\pi_h^{k-1}(\cdot|s)),$
where $D(\cdot||\cdot)$ denotes the KL divergence, and $Y_{k-1}$ is the dual variable that balances between reward maximization and constraint satisfaction. The KL regularizer encourages the updated policy to remain close to $\widetilde\pi_h^{k-1}$.

\noindent\textbf{Regularized Dual Update} (Line \ref{line:dual update}). 
Next, we present the dual update, inspired by \cite{yu2026primaldual}:
\begin{align}\label{eq:dual update}
\begin{aligned}
    Y_{k} &\leftarrow \big[(1-\alpha \eta H^3)Y_{k-1} \\
    &\quad+\eta(b-V_{{k-1},1}^g(s_1)- \alpha H^3 - 2\theta H^2)\big]_+.
\end{aligned}
\end{align}
The dual variable $Y_k$ increases when the estimated constraint value falls below the threshold $b$ and decreases otherwise. Thus, it adaptively balances reward maximization and constraint satisfaction.

For comparison, we recall the dual update used in \cite{ding2021provably, ghosh2022provably} for the fixed-reward setup:
\begin{align}\label{eq:dual previous}
Y_{k} \leftarrow \left[Y_{k-1} + \eta (b - V_{{k-1},1}^g(s_1))\right]_{[0, \frac{2}{\gamma}]}.
\end{align}
In this update, the dual variable is upper clipped by $2/\gamma$ to prevent it from diverging. The threshold $2/\gamma$ is chosen because it serves as a strict upper bound on the optimal dual variable under \Cref{assum:slater}. In contrast, our update~\eqref{eq:dual update} introduces additional regularizers $-\alpha \eta H^3 Y_{k-1}-\eta(\alpha H^3 + 2\theta H^2)$, instead of relying on upper clipping. Hence, our update does not require knowledge of $\gamma$.

\noindent\textbf{Parameter Estimation} (Lines \ref{line:param start} - \ref{line:param end}).
To obtain tighter confidence intervals, we extend the weighted ridge regression technique—previously used in unconstrained settings~\cite{zhou2021nearly}—to our constrained setting. This differs from \cite{ding2021provably}, which uses standard ridge regression and does not account for the conditional variance of the next-state value. In contrast, we employ weighted ridge regression–based parameter estimation, where the weights are chosen as estimates of this variance, enabling Bernstein-type concentration inequalities and yielding optimal dependence on $H$ in the regret bound.

Before describing the weighted ridge regression step, we introduce the basic structure of the $Q$-function estimates, denoted by $Q_{k,h}^r, Q_{k,h}^g$. These estimates are computed by backward induction from $h=H$ to $1$ as follows:
\begin{align}\label{eq:Q}
\begin{aligned}
    Q_{k,h}^r(\cdot,\cdot) &\leftarrow \bigg[ r_h^k(\cdot,\cdot)+ \langle\widehat\theta_{k,h}^r, \phi_{V_{k,h+1}^r}(\cdot,\cdot)\rangle \\
    &\qquad\qquad+ \widehat\beta_k \left\|\phi_{V_{k,h+1}^r}(\cdot,\cdot)\right\|_{(\widehat\Sigma_{k,h}^r)^{-1}}\bigg]_{[0,H-h+1]},\\
    Q_{k,h}^g(\cdot,\cdot) &\leftarrow \bigg[g_h(\cdot,\cdot)+ \langle\widehat\theta_{k,h}^g, \phi_{V_{k,h+1}^g}(\cdot,\cdot)\rangle \\
    &\qquad\qquad+ \widehat\beta_k \left\|\phi_{V_{k,h+1}^g}(\cdot,\cdot)\right\|_{(\widehat\Sigma_{k,h}^g)^{-1}}\bigg]_{[0,H-h+1]},
\end{aligned}
\end{align}
where the optimistic bonus parameter $\widehat \beta_k$ is given by
\begin{align*}
    \widehat\beta_k &= 8\sqrt{d\log(1+k/\lambda)\log(8Hk^2/\delta)} \\
    &\quad+ 4\sqrt{d}\log(8Hk^2/\delta)+\sqrt{\lambda}B.
\end{align*}

The intuition behind the $Q$-function estimates is as follows. Given a policy $\pi^k$, we expect that $Q_{k,h}^r(s,a) \approx Q_{h}^{r^k, \pi^k}(s,a) = r_h^k(s,a) + P_hV_{h+1}^{r^k,\pi^k}(s,a)$, where the equality follows from the Bellman equation. Assuming $\widehat \theta_{k,h}^r \approx \theta_h^*$, by the definition of linear mixture CMDPs (\Cref{def:linear mixture}), we have 
$\langle \widehat \theta_{k,h}^r, \phi_{V_{k,h+1}^r}(s,a) \rangle \approx \langle\theta_h^*, \phi_{V_{k,h+1}^r}(s,a)\rangle = \sum_{s'} \langle \theta_h^*, \phi(s'|s,a) \rangle V_{k,h+1}^r(s') = P_h V_{k,h+1}^r(s,a)$. 
Moreover, $\widehat\beta_k \|\phi_{V_{k,h+1}^r}(\cdot,\cdot)\|_{(\widehat \Sigma_{k,h}^r)^{-1}}$ serves as an optimistic bonus term to promote exploration. These arguments validate the design of $Q_{k,h}^r$ in \eqref{eq:Q}, and can be applied to $Q_{k,h}^g$ as well.

Now, we present how to obtain $\widehat \theta_{k,h}^r$ and $\widehat \theta_{k,h}^g$ through weighted ridge regression, which approximate $\theta_h^*$. For $\ell \in \{r,g\}$, we define
\begin{align} \label{eq:weighted ridge regression}
\begin{aligned}
\widehat\theta_{k,h}^\ell \leftarrow &\argmin_{\theta \in \bbR^d} \ \lambda\|\theta\|_2^2 \\
 &\quad+ \sum_{\tau=1}^{k-1} \frac{\big[\langle \phi_{V_{\tau,h+1}^\ell}(s_h^\tau,a_h^\tau), \theta\rangle - V_{\tau,h+1}^\ell(s_{h+1}^\tau)\big]^2}{(\bar\sigma_{\tau,h}^\ell)^2},    
\end{aligned}
\end{align}
where $(\bar\sigma_{\tau,h}^\ell)^2$ is an upper bound on $\bbV_h V_{\tau,h+1}^\ell(s_h^\tau,a_h^\tau)$. Equation \eqref{eq:weighted ridge regression} admits the following closed form: for $\ell \in \{r,g\}$,
\begin{align}\label{eq:theta hat}
\begin{aligned}
    &\widehat\theta_{k,h}^\ell \leftarrow \left(\widehat\Sigma_{k,h}^\ell\right)^{-1} \widehat b_{k,h}^\ell, \\
    &\widehat\Sigma_{k,h}^\ell \leftarrow  \lambda I +\sum_{\tau=1}^{k-1}\left(\bar\sigma_{\tau,h}^\ell\right)^{-2} \phi_{V_{\tau,h+1}^\ell}(s_h^\tau, a_h^\tau)\phi_{V_{\tau,h+1}^\ell}(s_h^\tau, a_h^\tau)^\top, \\
    &\widehat b_{k,h}^\ell \leftarrow \sum_{\tau=1}^{k-1}\left(\bar\sigma_{\tau,h}^\ell\right)^{-2}\phi_{V_{\tau,h+1}^\ell}(s_h^\tau, a_h^\tau) V_{\tau,h+1}^\ell(s_{h+1}^\tau).
\end{aligned}
\end{align}

We emphasize that the weighted ridge regression in \eqref{eq:weighted ridge regression} is useful for handling heteroscedastic noise induced by the time-inhomogeneous transition kernel. Since $P_h V_{k,h+1}^\ell(s,a)$ is not directly observable, we estimate $\widehat \theta_{k,h}^\ell$ by regressing the samples $V_{\tau,h+1}^\ell(s_{h+1}^\tau)$. This introduces a noise term $V_{\tau,h+1}^\ell(s_{h+1}^\tau) - P_h V_{\tau,h+1}^\ell(s_h^\tau,a_h^\tau)$, whose conditional variance $\bbV_h V_{\tau,h+1}^\ell(s_h^\tau,a_h^\tau)$ depends on $h$, resulting in heteroscedastic noise. In such cases, standard ridge regression is statistically inefficient due to the heteroscedastic structure~\cite{kirschner2018information}. In contrast, weighted ridge regression effectively accounts for this heteroscedasticity~\cite{zhou2021nearly}, enabling the use of Bernstein-type concentration inequalities and yielding tighter confidence intervals. 

One remaining question for the parameter estimation step is how to construct $(\bar\sigma_{k,h}^r)^2,(\bar\sigma_{k,h}^g)^2$, which are upper bounds on $\bbV_{h}V_{k,h+1}^r(s_h^k,a_h^k),\bbV_{h}V_{k,h+1}^g(s_h^k,a_h^k)$, respectively. Again, since $P_h$ is unknown, we construct the following variance estimate, motivated by the relation that $\bbV_h V_{k,h+1}^\ell(s,a) = \langle \phi_{(V_{k,h+1}^\ell)^2}(s,a), \theta_h^* \rangle -  \langle \phi_{V_{k,h+1}^\ell}(s,a), \theta_h^* \rangle^2$: for $\ell \in \{r,g\}$,
\begin{align}\label{eq:bar bbV}
\begin{aligned}
    &\bar\bbV_h V_{k,h+1}^\ell(\cdot,\cdot) = \left[\langle \phi_{(V_{k,h+1}^\ell)^2}, \widetilde\theta_{k,h}^\ell(\cdot,\cdot) \rangle\right]_{\left[0, H^2\right]}\\
    &\qquad\qquad\qquad- \left[\langle \phi_{V_{k,h+1}^\ell}, \widehat\theta_{k,h}^\ell(\cdot,\cdot) \rangle\right]_{\left[0, H\right]}^2.
\end{aligned}
\end{align}
where $\widehat \theta_{k,h}^\ell$ is defined in \eqref{eq:theta hat} and 
\begin{align}\label{eq:theta tilde Sigma tilde b tilde}
\begin{aligned}
    &\widetilde\theta_{k,h}^\ell \leftarrow \left(\widetilde\Sigma_{k,h}^\ell\right)^{-1} \widetilde b_{k,h}^\ell,\\
    &\widetilde\Sigma_{k,h}^\ell \leftarrow \lambda I +\sum_{\tau=1}^{k-1} \phi_{(V_{\tau,h+1}^\ell)^2}(s_h^\tau, a_h^\tau)\phi_{(V_{\tau,h+1}^\ell)^2}(s_h^\tau, a_h^\tau)^\top ,\\
    &\widetilde b_{k,h}^\ell \leftarrow \sum_{\tau=1}^{k-1}\phi_{(V_{\tau,h+1}^\ell)^2}(s_h^\tau, a_h^\tau)(V_{\tau,h+1}^\ell)^2(s_{h+1}^\tau).
\end{aligned}
\end{align}
Note that \eqref{eq:theta tilde Sigma tilde b tilde} is a consequence of standard ridge regression with respect to $(V_{k,h+1}^\ell)^2$, so that $\langle \phi_{(V_{k,h+1}^\ell)^2}, \widetilde\theta_{k,h}^\ell(s,a) \rangle \approx P_h(V_{k,h+1}^\ell)^2(s,a)$. Moreover, the following proposition characterizes the discrepancy between the variance estimated in \eqref{eq:bar bbV} and its true value $\bbV_h V_{k,h+1}^\ell(\cdot,\cdot)$.
\begin{proposition}\label{prop: E step 1}
    For any $(s,a,h,k,\ell) \in  S  \times  A  \times [H] \times [K] \times \{r,g\}$,
    {\small\begin{align*}
        &\left|\bar\bbV_h V_{k,h+1}^\ell(s,a) - \bbV_h V_{k,h+1}^\ell(s,a)\right| \\
        &\leq \min\left\{H^2, \left\|\phi_{(V_{k,h+1}^\ell)^2}(s,a)\right\|_{(\widetilde\Sigma_{k,h}^\ell)^{-1}} \left\|\widetilde\theta_{k,h}^\ell - \theta_h^*\right\|_{\widetilde\Sigma_{k,h}^\ell}\right\} \\
        &+ \min\left\{H^2, 2H\left\|\phi_{V_{k,h+1}^\ell}(s,a)\right\|_{\left(\widehat \Sigma_{k,h}^\ell\right)^{-1}}  \left\|\widehat\theta_{k,h}^\ell - \theta_h^*\right\|_{\widehat\Sigma_{k,h}^\ell}\right\}.
    \end{align*}}
\end{proposition}
Combining this with Theorem 4.1 in \cite{zhou2021nearly} implies that $\|(\widetilde\theta_{k,h}^\ell - \theta_h^*)\|_{\widetilde\Sigma_{k,h}^\ell}\leq \widetilde\beta_k,  \ \|\widehat\theta_{k,h}^\ell - \theta_h^*\|_{\widehat\Sigma_{k,h}^\ell} \leq  \check\beta_k,$
where 
\begin{align*}
 \widetilde\beta_k &=8H^2 \sqrt{d \log(1 + kH^4/(d\lambda)) \log(8Hk^2/\delta)} \\
 &\quad+ 4H^2 \log(8Hk^2/\delta) + \sqrt{\lambda}B,\\
 \check\beta_k &= 8d\sqrt{\log(1 + k/\lambda) \log(8Hk^2/\delta)} \\
 &\quad+ 4\sqrt{d} \log(8Hk^2/\delta) + \sqrt{\lambda}B.
\end{align*}

Based on these observations, we have $|\bar \bbV_h V_{k,h+1}^\ell(s,a) - \bbV_{h}V_{k,h+1}^\ell(s,a)| \leq E_{k,h}^\ell$, where the offset term $E_{k,h}^\ell$ is defined as follows: for $\ell \in \{r,g\}$,
\begin{align}\label{eq:E}
\begin{aligned}
    E_{k,h}^\ell&= \min\left\{H^2, \widetilde\beta_k \left\|\phi_{(V_{k,h+1}^\ell)^2}(s,a)\right\|_{(\widetilde\Sigma_{k,h}^\ell)^{-1}}\right\} \\
    &\quad+ \min\left\{H^2, 2H \check\beta_k \left\|\phi_{V_{k,h+1}^\ell}(s,a)\right\|_{\left(\widehat \Sigma_{k,h}^\ell\right)^{-1}}\right\}.
\end{aligned}
\end{align}
Finally, we present $\bar\sigma_{k,h}^r$ and $\bar\sigma_{k,h}^g$, defined for $\ell \in \{r,g\}$ as 
\begin{align}\label{eq: bar sigma}
    (\bar\sigma_{k,h}^\ell)^2 = \max\left\{H^2/d, \bar\bbV_h V_{k,h+1}^\ell(s_h^k,a_h^k) + E_{k,h}^\ell\right\}.
\end{align}

The effectiveness of our estimates is supported by the following lemma.
\begin{lemma}\label{lem: beta hat}
    For any $(k,h,\ell) \in [K] \times [H] \times \{r,g\}$, with probability at least $1-3\delta$,
    \begin{align*}
        &\left\|\theta_h^* - \widehat\theta_{k,h}^\ell\right\|_{\widehat\Sigma_{k,h}^\ell} \leq \widehat \beta_k, \ |\bbV_{h} V_{k,h+1}^\ell(s,a) - \bar\bbV_{h} V_{k,h+1}^\ell(s,a)|\leq E_{k,h}^\ell.
    \end{align*}
\end{lemma}

\begin{remark}
    The computational complexity of PD-POWERS is comparable to that of its unconstrained counterpart (POWERS, \cite{he2022near}). Specifically, it is $O(\min\{d^3 H K^2 | A |,\ | S || A |K\} + d^3 H K)$ with $O(HK)$ calls to the integrating oracle $\calO$ for computing $\sum_{s^{\prime}} \psi(s^{\prime}) V(s^{\prime})$, where $\psi: S  \to \bbR^d$ satisfies $\phi(s^{\prime}|s,a) = \psi(s^{\prime}) \odot \mu(s,a)$ for some $\mu: S  \times  A  \to \bbR^d$, and $\odot$ denotes the component-wise product. For more details, we refer the reader to \cite{he2022near}.
\end{remark}

\section{Analysis}\label{sec:analysis}
In this section, we present our main result (\Cref{thm:main}), which establishes upper bounds on both the regret and the constraint violation. We also provide a high-level overview of the analysis; the full proofs are provided in the appendix.
\begin{theorem}\label{thm:main}
    Suppose that \Cref{assum:slater} holds and $K \geq \max\{2H,H^2, d^3 H^3\}$. Set $\lambda = 1/B^2$, $\alpha = 1/(H^2\sqrt{K})$, $\eta = 1/(H\sqrt{K})$, and $\theta = 1/K$. With probability at least $1-6\delta$,
    \begin{align*}
        &\Regret(K) = \bigO\left(\sqrt{dH^4 K} + \sqrt{d^2 H^3 K} + H^5\sqrt{K}/\gamma^2\right), \\
        &\Violation(K) = \bigO\left(\sqrt{dH^4 K} + \sqrt{d^2 H^3 K}+ H^3\sqrt{K}/\gamma\right).
    \end{align*}
\end{theorem}

\subsection{Limitations of Strong Duality-Based Analysis}\label{subsec:dual}

Before presenting the overview of our analysis, we explain why the strong duality-based approach induced by \eqref{eq:dual previous} fails when the rewards vary across episodes.

On top of \eqref{eq:dual previous}, the strong duality–based analysis~\cite{ding2021provably, ghosh2022provably} begins by upper bounding the following composite regret term for any $Y \in [0,2/\gamma]$:
\begin{align}
\label{eq:composite regret}
\sum_{k=1}^K \big(V_1^{r,\pi^*}(s_1) - V_1^{r,\pi^k}(s_1)\big) + Y \sum_{k=1}^K \big(b - V_1^{g,\pi^k}(s_1)\big).
\end{align}

The next step is to separate the composite regret into regret and violation terms. Here, \cite{ding2020natural} shows that the following holds under strong duality: for an optimal policy $\pi^* \in \argmax_\pi V_1^{r,\pi}(s_1)\ \textnormal{s.t.}\ V_1^{g,\pi}(s_1)\geq b$ and some $\Delta_k$, if
\begin{equation}\label{eq:before separate}
    V_1^{r,\pi^*}(s_1) - V_1^{r,\pi^k}(s_1) + \frac{2}{\gamma}\big(b - V_1^{g,\pi^k}(s_1)\big) \leq \Delta_k,
\end{equation}
then
\begin{equation}\label{eq:separate}
V_1^{r,\pi^*}(s_1) - V_1^{r,\pi^k}(s_1) \leq \Delta_k
\ \ \text{and} \ \
b - V_1^{g,\pi^k}(s_1) \leq \frac{\gamma}{2}\Delta_k.
\end{equation}
Applying this argument for all $k \in [K]$ yields regret and violation bounds from a bound on \eqref{eq:composite regret}.

Now, let us examine whether the same argument can be applied in our setting. Since the rewards change over $k$, the analogue of \eqref{eq:before separate} becomes 
\[
    V_1^{r^k,\pi^{k,*}}(s_1) - V_1^{r^k,\pi^k}(s_1) + \frac{2}{\gamma}\big(b - V_1^{g,\pi^k}(s_1)\big) \leq \Delta_k,
\]
where $\pi^{k,*} \in \argmax_\pi V_1^{r^k,\pi}(s_1) \ \ \textnormal{s.t.} \ \ V_1^{g,\pi}(s_1)\geq b$ for each $k$. 
Consequently, as the analogue of \eqref{eq:composite regret}, the corresponding composite regret must use the sequence $\{\pi^{k,*}\}_{k\in [K]}$ as comparator policies, namely,
\[
    \sum_{k=1}^K \big(V_1^{r^k,\pi^{k,*}}(s_1) - V_1^{r^k,\pi^k}(s_1)\big) + Y \sum_{k=1}^K \big(b - V_1^{g,\pi^k}(s_1)\big).
\]
However, obtaining an upper bound on this term is non-trivial unless $r^k$ is known before episode $k$, which is not allowed in the adversarial-reward setting.
This highlights the limitations of strong duality-based analysis in our setting.

\subsection{Our Analysis}
To overcome these limitations, we establish a drift-based analysis based on \eqref{eq:dual update}, which does not rely on strong duality. As a first step, we provide decompositions of $\Regret(K)$ and $\Violation(K)$. 
Since the optimal policy $\pi^*$ satisfies $V_{1}^{g,\pi^*}(s_1) \geq b$ and $Y_k \geq 0$, we have
\begin{align*}
    &\Regret(K) \\
    &\leq \sum_{k=1}^K \left(V_1^{r^k,\pi^*}(s_1) + Y_k V_1^{g,\pi^*}(s_1) - V_{k,1}^r(s_1) - Y_k V_{k,1}^g(s_1) \right) \\
    &\quad+ \sum_{k=1}^K \left(V_{k,1}^r(s_1) - V_1^{r^k,\pi^k}(s_1)\right) + \sum_{k=1}^K Y_k\left(V_{k,1}^g(s_1)-b\right),\\
    &\Violation(K) 
    = \sum_{k=1}^K (b - V_{k,1}^{g}(s_1)) + \sum_{k=1}^K (V_{k,1}^{g}(s_1) - V_1^{g,\pi^k}(s_1)).
\end{align*}

Since we employ weighted ridge regression for parameter estimation, 
the bias terms appearing in both the regret and violation decompositions admit the following bounds.
\begin{lemma} \label{lem:bias}
    Suppose that $K\geq d^3 H^3$. With probability at least $1-6\delta$,
    \begin{align*}
        &\sum_{k=1}^K \left(V_{k,h}^r(s_h^k) - V_{h}^{r^k,\pi^k}(s_h^k)\right)=\bigO\left(\sqrt{dH^4 K} + \sqrt{d^2 H^3 K}\right) ,\\
        &\sum_{k=1}^K \left(V_{k,h}^g(s_h^k) - V_{h}^{g,\pi^k}(s_h^k)\right) = \bigO\left(\sqrt{dH^4 K} + \sqrt{d^2 H^3 K}\right).
    \end{align*}
\end{lemma}

Next, we focus on $\sum_{k=1}^K (b - V_{k,1}^{g}(s_1))$. By the dual update,
\begin{align*}
    \sum_{k=1}^K (b- V_{k,1}^g(s_1)) 
    &\leq \frac{Y_{K+1}}{\eta} + \alpha H^3 \sum_{k=1}^K Y_k + \alpha KH^3 + 2\theta KH^2.
\end{align*}
Therefore, it suffices to bound $Y_k$ uniformly over $k$. To this end, we employ drift-based arguments~\cite{yu2017online, wei2020online}. 
Based on the observation that our regularizer in \eqref{eq:dual update} satisfies 
$-\alpha\eta H^2(1+Y_k) - 2\eta\theta H \leq \eta\langle\pi_h^{k+1} - \pi_h^{k}, Q_{k,h}^g\rangle$, 
which follows from \cite{yu2026primaldual} and remains valid under our setting, we establish the following drift inequality:
\begin{align}
    &\frac{Y_{k+1}^2 - Y_{k}^2}{2} 
    \leq -\gamma \eta Y_k  + C \label{eq:drift maintext}\\
    &+\frac{\eta}{\alpha} \bbE_{\bar\pi}\left[\sum_{h=1}^H D(\bar\pi_h(\cdot|s_h)||\widetilde\pi_h^k(\cdot|s_h)) - D(\bar\pi_h(\cdot|s_h)|| \pi_h^{k+1}(\cdot|s_h))\right] \notag
\end{align}
where $C =  \frac{\alpha \eta H^3}{2} + 2\eta H^2\theta + 2\eta H^2 +2\eta^2(H^2 + \alpha^2 H^6 + 9 \eta^2 \alpha^2 H^8K^2 + 4\theta^2 H^4).$ 

From \eqref{eq:drift maintext}, we observe that $\widetilde{\pi}_h^k$ is a perturbed policy, and hence the KL divergence term $D(\bar\pi_h(\cdot|s_h) \| \widetilde\pi_h^k(\cdot|s_h))$ admits a small upper bound. This highlights the necessity of the policy perturbation step. Without this perturbation, the KL divergence term $D(\bar\pi_h(\cdot|s_h)\|\pi_h^k(\cdot|s_h))$ could be unbounded, as $\pi_h^k$ approaches the boundary of the simplex. With further analysis, we obtain the following bound.

\begin{lemma} \label{lem:Yk bound}
Suppose that \Cref{assum:slater} holds and $K \geq \max\{2H, H^2\}$. With probability at least $1-6\delta,$
    \[
        Y_k = \bigO(H^2/\gamma), \quad \forall k\in[K].
    \]
\end{lemma}

Since $Y_k$ is bounded by \Cref{lem:Yk bound}, it follows that $\sum_{k=1}^K Y_k\left(V_{k,1}^g(s_1)-b\right) = \bigO\left(H\sqrt{K} + H^5\sqrt{K}/\gamma^2\right).$

To complete the analysis, it remains to bound $\sum_{k=1}^K \left(V_1^{r^k,\pi^*}(s_1) + Y_k V_1^{g,\pi^*}(s_1) - V_{k,1}^r(s_1) - Y_k V_{k,1}^g(s_1) \right)$. By \Cref{lem: beta hat} and the value difference lemma (Lemma 1 of \cite{shani2020optimistic}), this term can be bounded as
\[
    \sum_{k=1}^K\bbE_{\pi^*}\left[\sum_{h=1}^H \langle Q_{k,h}^r(s_h,\cdot) + Y_kQ_{k,h}^g(s_h,\cdot), \pi_h^*(\cdot|s_h) - \pi_h^k(\cdot|s_h) \rangle | s_1\right].
\]
Moreover, since we employ the policy optimization~\eqref{eq:policy optimization}, the above term can be effectively controlled. Specifically, we apply the standard OMD lemma~\cite{hazan2016introduction}, together with the policy perturbation step. Combining these ingredients, we obtain the following lemma.
\begin{lemma}\label{lem:omd} 
Let $K \geq \max\{2H, H^2\}$. With probability at least $1-6\delta$,
    \begin{align*}
        &\sum_{k=1}^K \left(V_1^{r^k,\pi^*}(s_1) + Y_k V_1^{g,\pi^*}(s_1) - V_{k,1}^r(s_1) - Y_k V_{k,1}^g(s_1) \right) \\
        &=\bigO\left(H^5 \sqrt{K} /\gamma^2 + H^3\sqrt{K} + H^4/\gamma\right).
    \end{align*}
\end{lemma}

\section{Numerical Experiments}\label{sec:numerical}
\begin{figure}[t]
    \centering
    \begin{subfigure}{0.33\textwidth}
        \centering
        \includegraphics[width=\linewidth]{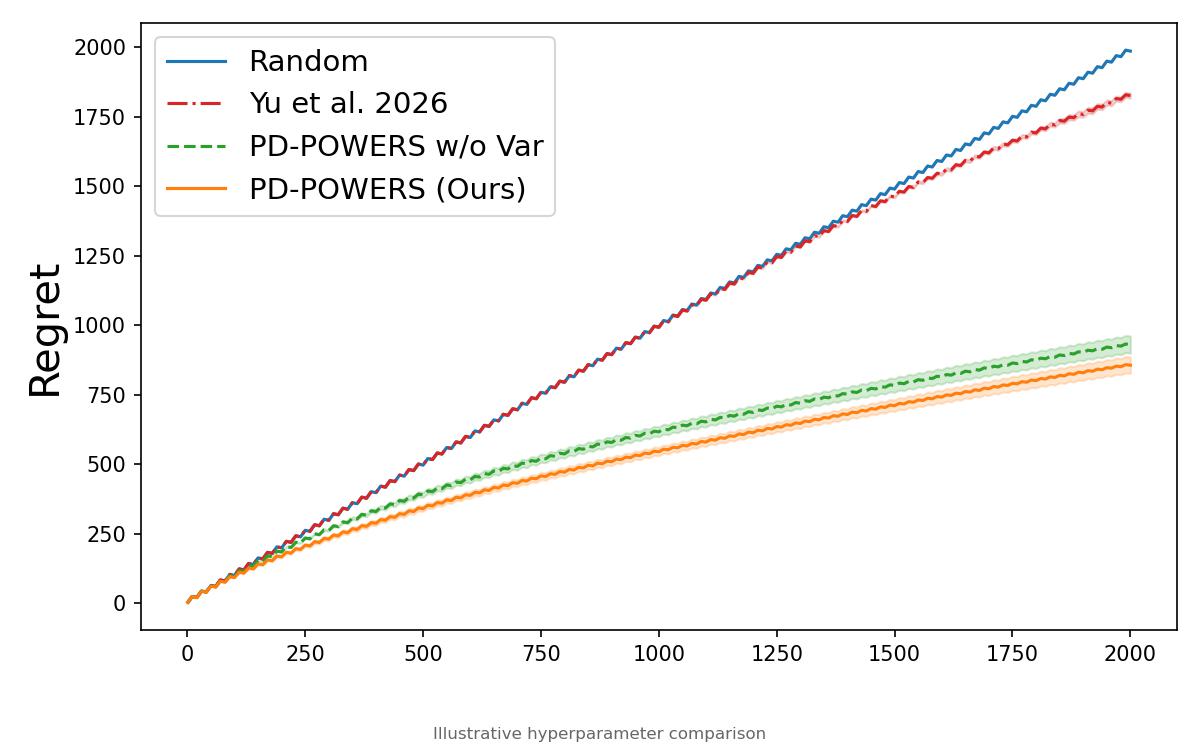}
    \end{subfigure}
    \hfill
    \begin{subfigure}{0.33\textwidth}
        \centering
        \includegraphics[width=\linewidth]{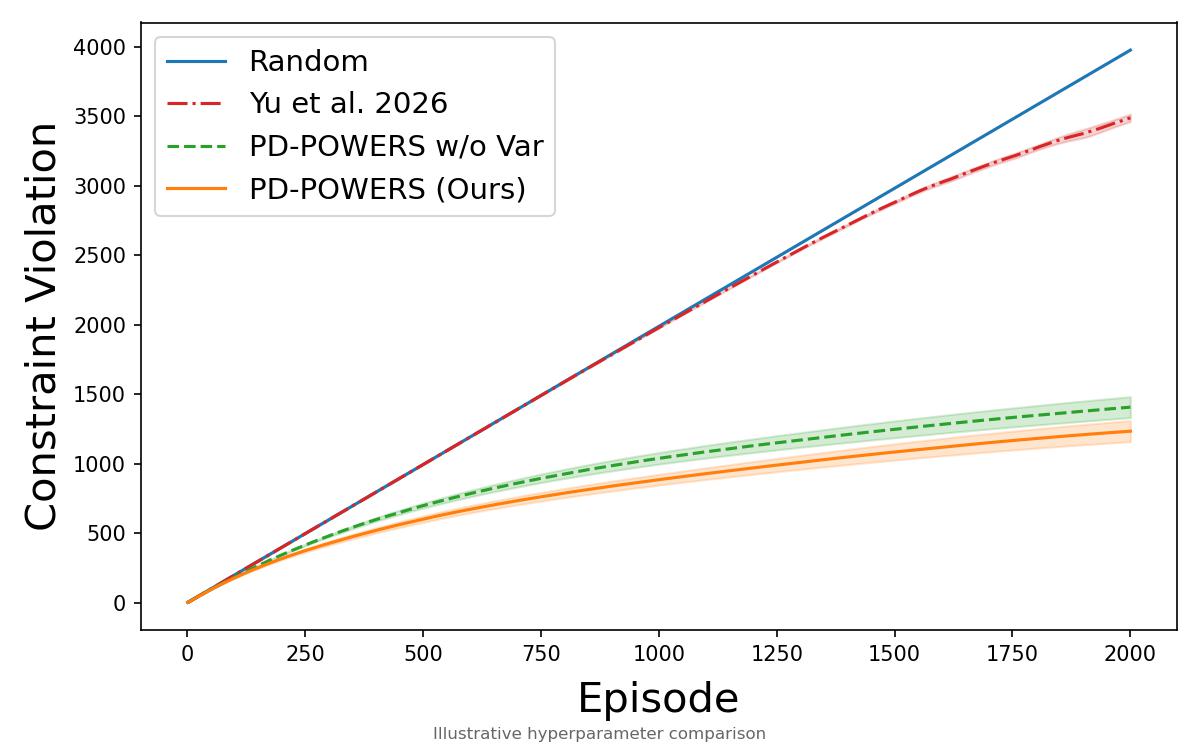}
    \end{subfigure}
    
    \caption{Regret and constraint violation of PD-POWERS.}
    \label{fig:main}
\end{figure}
{In Fig.~\ref{fig:main}, we evaluate PD-POWERS on a CMDP instance adapted from \cite{he2022near}, with modifications including the addition of a constraint. The details of the setting are described in the appendix. We conduct $20$ simulations with $K=2{,}000$ under different random seeds. Each plot shows the average, and the shaded regions indicate $95\%$ confidence intervals.

We compare PD-POWERS with three baselines. First, as in \cite{he2022near}, we consider a random policy that selects an action uniformly at random at each step. Second, to examine the effect of variance-aware parameter estimation, we consider an ablation of PD-POWERS that replaces the weighted ridge regression with standard unweighted ridge regression while keeping the remaining algorithmic components unchanged. Finally, we compare with the primal--dual policy optimization algorithm of \cite{yu2026primaldual}. As shown in Fig.~\ref{fig:main}, PD-POWERS achieves the best overall performance in terms of both regret and constraint violation among the compared methods. These experiments provide empirical support for the effectiveness of PD-POWERS in this instance.}


\section{Limitations}
{
PD-POWERS has three limitations. First, without the integration oracle $\calO$, its computational complexity may depend on $|\calS|$. Second, the $\gamma$-dependent terms in our regret and violation bounds can dominate when the Slater constant is small. Third, $\Violation(K)$ permits cancellation across episodes. In contrast, the stronger hard constraint violation, $\Violation_{\text{Hard}}(K) := \sum_{k=1}^K [b - V_{1}^{g,\pi^k}(s_1)]_+$ disallows such cancellation. Improving the dependence on $\gamma$ and controlling hard constraint violation are important directions for future work.
}
\section{Conclusion}\label{sec:conclusion}
{In this paper, we study online linear mixture CMDPs with adversarial rewards under full-information feedback and a fixed constraint function. We propose PD-POWERS, a primal-dual policy optimization algorithm combining a regularized dual update and weighted ridge regression-based parameter estimation. Moreover, we show that PD-POWERS achieves a near-optimal regret bound in the high-dimensional regime. Our results demonstrate that effective constrained decision-making with provable performance and safety guarantees is possible even in large state spaces and adversarially changing environments. 
}

%% file: appendix.tex

\section{Auxiliary Notions}

\begin{definition}[High-probability good event]\label{def:good event}
    We define a high-probability good event $\calE$ as
    \begin{equation*}
        \calE = \calE_1 \cap \calE_2 \cap \calE_3 \cap \calE_4,
    \end{equation*}
    where $\calE_1, \calE_2, \calE_3, \calE_4$ hold when the statement of Lemmas \ref{lem: beta hat}, \ref{lem:ltv}, \ref{lem:mds term}, and \ref{lem:mds term for V - Q} hold, respectively.
\end{definition}

\begin{definition}[Filtration]\label{def:filtration}
    Let $\calF_{k,h}$ denote the $\sigma$-algebra that includes all randomness up to step $h$ and episode $k$, i.e.,
    \begin{equation}\label{eq:calF}
        \calF_{k,h} = \sigma\left\{\{(s_j^\tau, a_j^\tau)\}_{\tau\in [k-1], j \in [H]} \cup \{r^\tau\}_{\tau \in [k]} \cup \{(s_j^k,a_j^k)\}_{j\in [h]}\right\}.
    \end{equation}
    Let $\calG_{k,h}$ denote the $\sigma$-algebra that includes all randomness fixed before sampling  $a_h^k$, i.e., 
    \begin{equation}\label{eq:calG}
        \calG_{k,h} = \sigma\left\{\{(s_j^\tau, a_j^\tau)\}_{\tau\in [k-1], j \in [H]} \cup \{r^\tau\}_{\tau \in [k]} \cup \{(s_j^k,a_j^k)\}_{j\in [h-1]} \cup s_h^k\right\}.
    \end{equation}
\end{definition}
Since the adversarial reward in episode $k$ is determined at the beginning of the episode, $\{r_h^k\}_{h\in [H]}$ is $\calF_{k,1}$-measurable. Moreover, we note that, $\{\pi_h^k\}_{h\in[H]}$, $\{Q_{k,h}^r, Q_{k,h}^g\}_{h\in [H]}$, and $\{V_{k,h}^r, V_{k,h}^g\}_{h\in [H]}$ are $\calF_{k,1}$-measurable, since they are determined by $\{(s_j^\tau, a_j^\tau)\}_{j\in [H], \tau \in [k-1]}$ and $\{r^\tau\}_{\tau \in [k]}$.

\section{Proof of Lemma~\ref{lem: beta hat}} 
\begin{proposition}[Restatement of \Cref{prop: E step 1}]
    For any $(s,a,h,k,\ell) \in \calS \times \calA \times [H] \times [K] \times \{r,g\}$,
    \begin{align*}
        \left|\bar\bbV_h V_{k,h+1}^\ell(s,a) - \bbV_h V_{k,h+1}^\ell(s,a)\right|
        &\leq \min\left\{H^2, \left\|\left(\widetilde\Sigma_{k,h}^\ell\right)^{-1/2}\phi_{(V_{k,h+1}^\ell)^2}(s,a)\right\|_2 \left\|\left(\widetilde\Sigma_{k,h}^\ell\right)^{1/2}(\widetilde\theta_{k,h}^\ell - \theta_h^*)\right\|_2\right\} \\
        &\quad+ \min\left\{H^2, 2H\left\|\left(\widehat \Sigma_{k,h}^\ell\right)^{-1/2}\phi_{V_{k,h+1}^\ell}(s,a)\right\|_2  \left\|\left(\widehat\Sigma_{k,h}^\ell\right)^{1/2}(\widehat\theta_{k,h}^\ell - \theta_h^*)\right\|_2\right\}  
    \end{align*}
\end{proposition}
\begin{proof}
    Consider the case of $r$. By definition and the triangle inequality,
    \begin{align*}
    \begin{aligned}
        \left|\bar\bbV_h V_{k,h+1}^r(s,a) - \bbV_h V_{k,h+1}^r(s,a)\right| 
        &\leq \underbrace{\left|\left[\langle \phi_{(V_{k,h+1}^r)^2}(s,a), \widetilde\theta_{k,h}^r \rangle\right]_{\left[0, H^2\right]} - \langle \phi_{(V_{k,h+1}^r)^2}(s,a), \theta_h^*\rangle\right|}_{\text{(a)}}\\
        &\quad+\underbrace{\left| \langle \phi_{V_{k,h+1}^r}(s,a), \theta_h^*\rangle^2 - \left[\langle \phi_{V_{k,h+1}^r}(s,a), \widehat\theta_{k,h}^r \rangle\right]_{\left[0, H\right]}^2 \right|}_{\text{(b)}}.
    \end{aligned}
    \end{align*}
    Term (a) can be bounded as
    \begin{align*}
    \begin{aligned}
        \text{Term (a)} &\leq \left|\langle \phi_{(V_{k,h+1}^r)^2}(s,a), \widetilde\theta_{k,h}^r \rangle - \langle \phi_{(V_{k,h+1}^r)^2}(s,a), \theta_h^*\rangle\right|\\
        &\leq \left\|\left(\widetilde\Sigma_{k,h}^r\right)^{-1/2}\phi_{(V_{k,h+1}^r)^2}(s,a)\right\|_2 \left\|\left(\widetilde\Sigma_{k,h}^r\right)^{1/2}(\widetilde\theta_{k,h}^r - \theta_h^*)\right\|_2
    \end{aligned}
    \end{align*}
    where the first inequality is due to the fact that $\langle \phi_{(V_{k,h+1}^r)^2}(s,a), \theta_h^* \rangle = P _h (V_{k,h+1}^r)^2(s,a) \leq H^2$, and the second inequality is due to the Cauchy-Schwarz inequality. Furthermore, since term $\text{(a)}\leq H^2$, we have the following upper bound.
    \[
        \text{Term (a)} \leq \min\left\{H^2,\left\|\left(\widetilde\Sigma_{k,h}^r\right)^{-1/2}\phi_{(V_{k,h+1}^r)^2}(s,a)\right\|_2 \left\|\left(\widetilde\Sigma_{k,h}^r\right)^{1/2}(\widetilde\theta_{k,h}^r - \theta_h^*)\right\|_2\right\}.
    \]
    
    Next, term (b) is bounded as
    \begin{align*}
    \begin{aligned}
        \text{Term (b)} &= \left|\left[\langle \phi_{V_{k,h+1}^r}(s,a), \widehat\theta_{k,h}^r\rangle\right]_{\left[0, H\right]} + \langle \phi_{V_{k,h+1}^r}(s,a), \theta_h^*\rangle\right| 
        \left|\left[\langle \phi_{V_{k,h+1}^r}(s,a), \widehat\theta_{k,h}^r \rangle\right]_{\left[0, H\right]} - \langle \phi_{V_{k,h+1}^r}(s,a), \theta_h^*\rangle\right|\\
        &\leq 2H \left|\left[\langle \phi_{V_{k,h+1}^r}(s,a), \widehat\theta_{k,h}^r \rangle\right]_{\left[0, H\right]} - \langle \phi_{V_{k,h+1}^r}(s,a), \theta_h^*\rangle\right| \\
        &\leq 2H\left|\langle \phi_{V_{k,h+1}^r}(s,a), \widehat\theta_{k,h}^r \rangle- \langle \phi_{V_{k,h+1}^r}(s,a), \theta_h^*\rangle\right| \\
        &\leq 2H\left\|\left(\widehat\Sigma_{k,h}^r\right)^{-1/2}\phi_{V_{k,h+1}^r}(s,a)\right\|_2  \left\|\left(\widehat\Sigma_{k,h}^r\right)^{1/2}(\widehat\theta_{k,h}^r - \theta_h^*)\right\|_2
    \end{aligned}
    \end{align*}
    where the first and second inequalities are due to $\langle \phi_{V_{k,h+1}^r}(s,a), \theta_h^*\rangle =  P _h V_{k,h+1}^r(s,a) \leq H$, and the last inequality is due to the Cauchy-Schwarz inequality. Again, since term (b) $\leq H^2$, we have 
    \[
    \text{Term (b)} \leq \min\left\{ H^2, 2H \left\|\left(\widehat\Sigma_{k,h}^r\right)^{-1/2}\phi_{V_{k,h+1}^r}(s,a)\right\|_2  \left\|\left(\widehat\Sigma_{k,h}^r\right)^{1/2}(\widehat\theta_{k,h}^r - \theta_h^*)\right\|_2 \right\}
    \]
    Combining those, we have the desired result. By applying the same argument to $g$, we conclude the proof.
\end{proof}

\begin{lemma}[Restatement of \Cref{lem: beta hat}]
    For any $(k,h,\ell) \in [H] \times [K] \times \{r,g\}$, with probability at least $1-3\delta$,
    $$\left\|\left(\widehat\Sigma_{k,h}^\ell\right)^{1/2}(\theta_h^* - \widehat\theta_{k,h}^\ell)\right\|_2 \leq \widehat \beta_k, \quad |\bbV_{h} V_{k,h+1}^\ell(s,a) - \bar\bbV_{h} V_{k,h+1}^\ell(s,a)|\leq E_{k,h}^\ell.$$
\end{lemma}
\begin{proof}
Consider the case of $r$. We prove the first statement as follows. Let
\begin{align*}
    &\widetilde\calC_{k,h}^r = \left\{ \theta_h^*\in \bbR^d:\left\|\left(\widetilde\Sigma_{k,h}^r\right)^{1/2}(\theta_h^* - \widetilde \theta_{k,h}^r)\right\|_2\leq \widetilde \beta_k\right\},\\
    &\check\calC_{k,h}^r = \left\{ \theta_h^*\in \bbR^d:\left\|\left(\widehat\Sigma_{k,h}^r\right)^{1/2}(\theta_h^* - \widehat \theta_{k,h}^r)\right\|_2\leq \check \beta_k\right\},\\
    &\calC_{k,h}^r = \left\{ \theta_h^*\in \bbR^d:\left\|\left(\widehat\Sigma_{k,h}^r\right)^{1/2}(\theta_h^* - \widehat \theta_{k,h}^r)\right\|_2\leq \widehat \beta_k\right\}.
\end{align*}
For a fixed $h\in [H]$, we first show that $\theta_h^* \in \widetilde\calC_{k,h}^r$ and $\theta_h^* \in \check \calC_{k,h}^r$ for all $k\in [K]$ with high probability. Based on these, we then show that $\theta_h^* \in \calC_{k,h}^r$.

\textbf{Proof of $\theta_h^* \in \widetilde\calC_{k,h}^r$. }
Fix $h \in [H]$. Take
    \begin{align*}
        &\mu^* = \theta_h^*,\ x_k = \phi_{(V_{k,h+1}^r)^2}(s_h^k, a_h^k),\ y_k = (V_{k,h+1}^r)^2(s_{h+1}^k),\\
        &\eta_k = (V_{k,h+1}^r)^2(s_{h+1}^k) - \langle \phi_{(V_{k,h+1}^r)^2}(s_{h}^k, a_h^k),\theta_h^*\rangle,\\
        &Z_k = \widetilde \Sigma_{k,h}^r = \lambda I + \sum_{i=1}^{k-1} \phi_{(V_{i,h+1}^r)^2}(s_h^i, a_h^i) \phi_{(V_{i,h+1}^r)^2}(s_h^i, a_h^i)^\top, \\
        &\mu_k = \widetilde \theta_{k,h} = \left(\widetilde \Sigma_{k,h}^r\right)^{-1} \left(\sum_{i=1}^{k-1} (V_{i,h+1}^r)^2(s_{h+1}^i) \phi_{(V_{i,h+1}^r)^2}(s_h^i, a_h^i)\right).
    \end{align*}
    Here, $V_{k,h+1}^r$ is $\calF_{k,1}$-measurable, $s_h^k, a_h^k$ are $\calF_{k,h}$-measurable. So, $x_k$ is $\calF_{k,h}$-measurable. Since $s_{h+1}^k$ is $\calF_{k,h+1}$-measurable, so is $\eta_k$. Furthermore, $\calF_{k,h+1}\subset \calF_{k+1,h}$, $\eta_k$ is $\calF_{k+1,h}$-measurable. Also, $\bbE[\eta_k\mid \calF_{k,h}] = 0$. Now, we apply \Cref{lem:zhou thm 4.1} with the following parameters.
    \begin{align*}
    \begin{aligned}
        \|x_k\|_2 \leq H^2,\ |\eta_k| \leq  H^2, \bbE\left[\eta_k^2 |\calF_{k,h}\right]\leq H^4.
    \end{aligned}
    \end{align*}
    Hence, with probability at least $1-\delta/(2H)$,
    \begin{equation}\label{eq:beta tilde}
        \forall k\geq 2,\ \left\|\left(\widetilde\Sigma_{k,h}^r\right)^{1/2}(\theta_h^* - \widetilde \theta_{k,h}^r)\right\|_2\leq \widetilde \beta_k.
    \end{equation}
    where 
    \[
    \widetilde\beta_k= 8H^2 \sqrt{d \log(1 + kH^4/(d\lambda)) \log(8Hk^2/\delta)} + 4H^2 \log(8Hk^2/\delta) + \sqrt{\lambda}B.
    \]
\textbf{Proof of $\theta_h^* \in \check\calC_{k,h}^r$. }
Fix $h \in [H]$. Take
    \begin{align*}
        &\mu^* = \theta_h^*,\ x_k = \phi_{V_{k,h+1}^r}(s_h^k, a_h^k)/\bar\sigma_{k,h}^r,\ y_k = V_{k,h+1}^r(s_{h+1}^k)/\bar\sigma_{k,h}^r,\\
        &\eta_k = \left(V_{k,h+1}^r(s_{h+1}^k) - \langle \phi_{V_{k,h+1}^r}(s_{h}^k, a_h^k),\theta_h^*\rangle\right) / \bar\sigma_{k,h}^r,\\
        &Z_k =\widehat\Sigma_{k,h}^r = \lambda I + \sum_{i=1}^{k-1} \phi_{V_{i,h+1}^r}(s_h^i, a_h^i) \phi_{V_{i,h+1}^r}(s_h^i, a_h^i)^\top / (\bar\sigma_{i,h}^r)^2, \\
        &\mu_k = \widehat \theta_{k,h}^r = \left(\widehat\Sigma_{k,h}^r\right)^{-1} \left(\sum_{i=1}^{k-1} V_{i,h+1}^r(s_{h+1}^i) \phi_{V_{i,h+1}^r}(s_h^i, a_h^i) / (\bar\sigma_{i,h}^r)^2\right).
    \end{align*}
    Here, $V_{k,h+1}^r$ is $\calF_{k,1}$-measurable, $\bar\sigma_{k,h}^r, s_h^k, a_h^k$ are $\calF_{k,h}$-measurable. So, $x_k$ is $\calF_{k,h}$-measurable. Since $s_{h+1}^k$ is $\calF_{k,h+1}$-measurable, so is $\eta_k$. Furthermore, $\calF_{k,h+1}\subset \calF_{k+1,h}$, $\eta_k$ is $\calF_{k+1,h}$-measurable. Also, $\bbE[\eta_k\mid \calF_{k,h}] = 0$. Now, we apply \Cref{lem:zhou thm 4.1} with the following parameters.
    \begin{align*}
    \begin{aligned}
        \|x_k\|_2 \leq H / \sqrt{H^2 / d}=\sqrt{d},\ |\eta_k| \leq  H / \sqrt{H^2 / d}=\sqrt{d},\ \bbE\left[\eta_k^2 | \calF_{k,h}\right]\leq d.
    \end{aligned}
    \end{align*}
    Hence, with probability at least $1-\delta/(2H)$,
    \begin{equation}\label{eq:beta check}
        \forall k\geq 2,\ \left\|\left(\widehat\Sigma_{k,h}^r\right)^{1/2}(\theta_h^* - \widehat \theta_{k,h}^r)\right\|_2\leq \check \beta_k.
    \end{equation}
    where 
    \[
    \check\beta_k= 8d\sqrt{\log(1 + k/\lambda) \log(8Hk^2/\delta)} + 4\sqrt{d} \log(8Hk^2/\delta) + \sqrt{\lambda}B.
    \]

\paragraph{Proof of $\theta_h^* \in \calC_{k,h}^r$.}
Fix $h\in [H]$. Take
\begin{align}\label{eq:calC parameters}
\begin{aligned}
    &\mu^* = \theta_h^*,\ x_k = \phi_{V_{k,h+1}^r}(s_h^k, a_h^k)/\bar\sigma_{k,h}^r,\ y_k = \langle \theta_h^*, x_k\rangle + \eta_k,\\
    &\eta_k = \mathds{1}\{\theta_h^* \in \widetilde\calC_{k,h}^r \cap \check\calC_{k,h}^r\}\left(V_{k,h+1}^r(s_{h+1}^k) - \langle \phi_{V_{k,h+1}^r}(s_{h}^k, a_h^k), \theta_h^* \rangle\right) / \bar\sigma_{k,h}^r,\\
    &Z_k = \widehat\Sigma_{k,h}^r= \lambda I + \sum_{i=1}^{k-1} \phi_{V_{i,h+1}^r}(s_h^i, a_h^i) \phi_{V_{i,h+1}^r}(s_h^i, a_h^i)^\top / (\bar\sigma_{i,h}^r)^2, \ \mu_k = Z_k^{-1} \sum_{i=1}^{k-1} x_i y_i.
\end{aligned}
\end{align}
Note that $\mathds{1}\{\theta_h^* \in \widetilde \calC_{k,h}^r \cap \check\calC_{k,h}^{r}\}$ is $\calF_{k,h}$-measurable. By applying the same argument in the proof of $\theta_h^* \in \check\calC_{k,h}^r$, it follows that
\begin{align*}
    \bbE[\eta_k^2 | \calF_{k,h}] 
    &\leq \mathds{1}\{\theta_h^* \in \widetilde \calC_{k,h}^r \cap \check\calC_{k,h}^{r}\} (\bar\sigma_{k,h}^r)^{-2} [\bbV_h V_{k,h+1}^{r}](s_h^k,a_h^k)\\
    &\leq \mathds{1}\{\theta_h^* \in \widetilde \calC_{k,h}^r \cap \check\calC_{k,h}^{r}\} (\bar\sigma_{k,h}^r)^{-2}\cdot 
    \Bigg(\bar\bbV_h V_{k,h+1}^r(s_h^k,a_h^k)\\
    &\quad+ \min\left\{H^2, \left\|\left(\widetilde\Sigma_{k,h}^r\right)^{-1/2}\phi_{(V_{k,h+1}^r)^2}(s_h^k,a_h^k)\right\|_2 \left\|\left(\widetilde\Sigma_{k,h}^r\right)^{1/2}(\widetilde\theta_{k,h} - \theta_h^*)\right\|_2\right\} \\
    &\quad+ \min\left\{H^2, 2H\left\|\left(\widehat \Sigma_{k,h}^r\right)^{-1/2}\phi_{V_{k,h+1}^r}(s_h^k,a_h^k)\right\|_2  \left\|\left(\widehat\Sigma_{k,h}^r\right)^{1/2}(\widehat\theta_{k,h} - \theta_h^*)\right\|_2\right\} \Bigg)\\
    &\leq (\bar\sigma_{k,h}^r)^{-2}\cdot 
    \Bigg(\bar\bbV_h V_{k,h+1}^r(s_h^k,a_h^k) + \min\left\{H^2, \widetilde\beta_k \left\|\left(\widetilde\Sigma_{k,h}^r\right)^{-1/2}\phi_{(V_{k,h+1}^r)^2}(s_h^k,a_h^k)\right\|_2\right\} \\
    &\quad+ \min\left\{H^2, \check\beta_k  \left\|\left(\widehat \Sigma_{k,h}^r\right)^{-1/2}\phi_{V_{k,h+1}^r}(s_h^k,a_h^k)\right\|_2\right\} \Bigg)\\
    &\leq  1
\end{align*}
where the second inequality follows from \Cref{prop: E step 1}, the second inequality is due to $\mathds{1}\{\theta_h^* \in \widetilde \calC_{k,h}^r \cap \check\calC_{k,h}^{r}\}$, and the last inequality is due to the definition of $\bar\sigma_{k,h}^r$. Again, by \Cref{lem:zhou thm 4.1}, with probability at least $1-\delta / (2H)$
\begin{align}\label{eq:beta hat}
    \left\|\left(\widehat\Sigma_{k,h}^r\right)^{1/2}(\theta_h^* - \mu_k)\right\|_2 \leq \widehat\beta_{k,h}
\end{align}
where $\mu_k$ is defined in \eqref{eq:calC parameters} and
\begin{align*}
    \widehat\beta_k = 8\sqrt{d\log(1+k/\lambda)\log(8Hk^2/\delta)} + 4\sqrt{d}\log(8Hk^2/\delta)+\sqrt{\lambda}B.
\end{align*}

By union bound, \eqref{eq:beta tilde}, \eqref{eq:beta check}, and \eqref{eq:beta hat} hold with probability at least $1-3\delta/(2H)$. On this event, since $\mathds{1}\{\theta_h^* \in \widetilde \calC_{k,h}^r \cap \check\calC_{k,h}^{r}\} = 1$, $y_k$ and $\mu_k$ defined in \eqref{eq:calC parameters} become
\begin{align*}
    y_k = V_{k,h+1}^r(s_{h+1}^k) / \bar\sigma_{k,h}^r, \ \mu_k = \widehat \theta_{k,h}^r.
\end{align*}
Thus, we have $\theta_h^* \in \calC_{k,h}^r$, as \eqref{eq:beta hat} is assumed to be true.
Moreover, we can apply the same argument to $g$. Finally, by taking union bound over $h \in [H]$ and $\ell \in \{r,g\}$, with probability at least $1-3\delta$, 
\begin{align}\label{eq:beta hat conclusion}
    \forall (k,h,\ell) \in [K] \times [H] \times \{r,g\}:\ \theta_h^* \in \calC_{k,h}^\ell.
\end{align}
Additionally, on the event \eqref{eq:beta tilde}, \eqref{eq:beta check}, and \eqref{eq:beta hat conclusion}, by \Cref{prop: E step 1}, the second statement of the lemma is proved.
\end{proof}

Additionally, we introduce the following lemma, which follows from \Cref{lem: beta hat}.
\begin{lemma}\label{lem:optimism 1}
    Suppose that the statement of \Cref{lem: beta hat} holds. For any $(s,a,h,k)$,
    \begin{align*}
        & r_h^k(s,a) +  P _hV_{k,h+1}^r(s,a) - Q_{k,h}^r(s,a) \leq 0, \ g_h(s,a) +  P _hV_{k,h+1}^g(s,a) - Q_{k,h}^g(s,a) \leq 0.
    \end{align*}
    Moreover, for all $(s,a,h,k) \in \calS\times\calA \times [H] \times [K]$,
    \begin{align*}
        &V_{k,h}^r(s) \geq V_{k,h}^{r^k, \pi^k}(s),\ Q_{k,h}^r(s,a) \geq Q_{k,h}^{r^k, \pi^k}(s,a), \\
        &V_{k,h}^g(s) \geq V_{k,h}^{g, \pi^k}(s),\ Q_{k,h}^g(s,a) \geq Q_{k,h}^{g, \pi^k}(s,a).
    \end{align*}
\end{lemma}
\begin{proof}
Note that
\begin{align*}
    r_h^k(s,a)+ \langle\widehat\theta_{k,h}^r, \phi_{V_{k,h+1}^r}(s,a)\rangle + \widehat\beta_k \left\|\left(\widehat\Sigma_{k,h}^r\right)^{-1/2} \phi_{V_{k,h+1}^r}(s,a)\right\|_2
    &\geq r_h^k(s,a) + [P_h V_{k,h+1}^r](s,a) + \widehat\beta_k \left\|\left(\widehat\Sigma_{k,h}^r\right)^{-1/2} \phi_{V_{k,h+1}^r}(s,a)\right\|_2 \\
    &\quad- \left\|\left(\widehat \Sigma_{k,h}^r\right)^{1/2}(\theta_h^* - \widehat \theta_{k,h}^r)\right\|_2
    \left\|\left(\widehat \Sigma_{k,h}^r\right)^{-1/2}\phi_{V_{k,h+1}^r}(s,a)\right\|_2 \\
    &\geq r_h^k(s,a) + [P_h V_{k,h+1}^r](s,a)
\end{align*}
where the first inequality is due to the triangle inequality and the Cauchy-Schwarz inequality, and the second inequality is due to \Cref{lem: beta hat}. Note that $r_h^k(s,a) + [P_h V_{k,h+1}^r](s,a) \leq H-h+1$. Then it follows that
\begin{align*}
    \left[r_h^k(s,a)+ \langle\widehat\theta_{k,h}^r, \phi_{V_{k,h+1}^r}(s,a)\rangle + \widehat\beta_k \left\|\left(\widehat\Sigma_{k,h}^r\right)^{-1/2} \phi_{V_{k,h+1}^r}(s,a)\right\|_2\right]_{[0,H-h+1]}\geq r_h^k(s,a) + [P_h V_{k,h+1}^r](s,a).
\end{align*}
Note that the left-hand side is equal to $Q_{k,h}^r(s,a)$. Moreover, we can apply the same argument to $g$. Then we conclude the proof of the first statement. 

Next, we prove the second statement by induction. Consider the case of $r$. For $h=H+1$, recall that $V_{k,H+1}^r(s) = V_{H+1}^{r^k, \pi^k}(s) = 0$ and $Q_{k,H+1}^r(s,a) = Q_{H+1}^{r^k, \pi^k}(s,a) = 0$. Now, suppose that $V_{k,h+1}^r(s) \geq V_{h+1}^{r^k, \pi^k}(s)$ and $Q_{k,h+1}^r(s,a) \geq  Q_{h+1}^{r^k, \pi^k}(s,a)$ for all $(s,a)$. It follows that
\begin{align*}
    Q_{k,h}^r(s,a) 
    &\geq r_h^k(s,a) + P_h V_{k,h+1}^r(s,a)\\ 
    &\geq r_h^k(s,a) + P_h V_{h+1}^{r^k,\pi^k}(s,a) \\
    &=Q_{h}^{r^k,\pi^k}(s,a)
\end{align*}
where the first inequality follows from the first statement, the second inequality follows from the induction hypothesis, and the equality is due to the Bellman equation. Furthermore, it follows that
\begin{align*}
    V_{k,h}^r(s)=\sum_a\pi_h^k(a|s)Q_{k,h}^r(s,a) \geq \sum_a\pi_h^k(a|s)Q_{h}^{r^k,\pi^k}(s,a) = V_{h}^{r^k,\pi^k}(s).
\end{align*}
This concludes the induction. Since the same argument can be applied to $g$, the proof is completed.
\end{proof}

\section{Proof of Lemma~\ref{lem:bias}}

\begin{lemma}\label{lem:ltv}
With probability at least $1-\delta$,
    \begin{align*}
        \sum_{k=1}^K \sum_{h=1}^H \bbV_h V_{h+1}^{r^k,\pi^k}(s_h^k, a_h^k) \leq 3(HT + H^3\log(2/\delta)),\quad
        \sum_{k=1}^K \sum_{h=1}^H \bbV_h V_{h+1}^{g,\pi^k}(s_h^k, a_h^k) \leq 3(HT + H^3\log(2/\delta)).
    \end{align*}
\end{lemma}
\begin{proof}
    The statement is proved by Lemma C.5 in \cite{jin2018q}.
\end{proof}

\begin{lemma}[Lemma C.3 of \cite{zhou2021nearly}]\label{lem:mds term}
    For any $h\in[H]$, with probability at least $1-\delta$,
    \begin{align*}
        \sum_{k=1}^K \sum_{j=h}^H\left[ P_j(V_{k,j+1}^r-V_{j+1}^{r^k,\pi^k})(s_j^k,a_j^k) - (V_{k,j+1}^r - V_{j+1}^{r^k,\pi^k})(s_{j+1}^k) \right] \leq 4H \sqrt{2T\log(2H/\delta)},\\
        \sum_{k=1}^K \sum_{j=h}^H\left[ P_j(V_{k,j+1}^g-V_{j+1}^{g,\pi^k})(s_j^k,a_j^k) - (V_{k,j+1}^g - V_{j+1}^{g,\pi^k})(s_{j+1}^k) \right] \leq 4H \sqrt{2T\log(2H/\delta)}.
    \end{align*}
\end{lemma}
\begin{proof}
    Consider the case of $r$. Recall that $\{r_h^k\}_{h\in [H]}$, $\{V_{k,h+1}^r\}_{h\in [H]}$, and $\{\pi_h^k\}_{h\in [H]}$ are $\calF_{k,1}$-measurable, $s_h^k, a_h^k$ are $\calF_{k,h}$-measurable, and $s_{h+1}^k$ is $\calF_{k,h+1}$-measurable. Let
    \[
        X_{k,h}=  P_h(V_{k,h+1}^r-V_{h+1}^{r^k,\pi^k})(s_h^k,a_h^k) - (V_{k,h+1}^r - V_{h+1}^{r^k,\pi^k})(s_{h+1}^k).
    \]
    For any $h\in[H]$, the following is a martingale difference sequence:
    $$\{X_{1,h}, \ldots, X_{1,H}, X_{2,h}, \ldots, X_{2,H}, \ldots, X_{K,h},\ldots, X_{K,H}\}.$$ 
    Moreover, we have
    $\left| P_h(V_{k,h+1}^r-V_{h+1}^{r^k,\pi^k})(s_h^k,a_h^k) - (V_{k,h+1}^r - V_{h+1}^{r^k,\pi^k})(s_{h+1}^k)\right| \leq 4H$.
    Then the Azuma-Hoeffding inequality implies that for a given $h\in[H]$, with probability at least $1-\delta/(2H)$,
    \begin{align*}
        \sum_{k=1}^K \sum_{j=h}^H  P_j(V_{k,j+1}^r-V_{j+1}^{r^k,\pi^k})(s_j^k,a_j^k) - (V_{k,j+1}^r - V_{j+1}^{r^k,\pi^k})(s_{j+1}^k) \leq  4H\sqrt{2T\log(2H/\delta)}.
    \end{align*}
    By union bound over $h\in [H]$, we can prove that the above inequality holds for all $j$ with probability at least $1-\delta/2$. Moreover, the same argument can be applied to $g$, and by union bound, the statement of the lemma holds with probability at least $1-\delta$.
\end{proof}

\begin{lemma}\label{lem:mds term for V - Q}
    For any $h \in [H]$, with probability at least $1-\delta$,
    \begin{align*}
        &\sum_{k=1}^K\sum_{j=h}^H \left[V_{k,j}^r(s_j^k) - V_j^{r^k,\pi^k}(s_j^k) - Q_{k,j}^r(s_j^k,a_j^k) + Q_j^{r^k,\pi^k}(s_j^k,a_j^k)\right]\leq 4H\sqrt{2T\log(2H/\delta)}, \\
        & \sum_{k=1}^K\sum_{j=h}^H \left[V_{k,j}^g(s_j^k) - V_j^{g,\pi^k}(s_j^k) - Q_{k,j}^g(s_j^k,a_j^k) + Q_j^{g,\pi^k}(s_j^k,a_j^k)\right]\leq 4H\sqrt{2T\log(2H/\delta)}.
    \end{align*}
\end{lemma}
\begin{proof}
Consider the case of $r$. Note that $\bbE[Q_{k,j}^r(s_j^k,a_j^k) - Q_j^{r^k,\pi^k}(s_j^k,a_j^k) | \calG_{k,j}] = V_{k,j}^r(s_j^k) - V_j^{r^k,\pi^k}(s_j^k)$, where the only randomness comes from $a_j^k \sim \pi_j^k(\cdot|s_j^k)$. Therefore, using the Azuma-Hoeffding inequality, it can be bounded with probability at least $1-\delta / (2H)$,
\begin{align*}
    \sum_{k=1}^K\sum_{j=h}^H \left[V_{k,j}^r(s_j^k) - V_j^{r^k,\pi^k}(s_j^k) - Q_{k,j}^r(s_j^k,a_j^k) + Q_j^{r^k,\pi^k}(s_j^k,a_j^k)\right] \leq 4H\sqrt{2T\log(2H/\delta)}.
\end{align*}
By union bound over $h\in [H]$, we can prove that the above inequality holds for all $h$ with probability at least $1-\delta/2$. Moreover, the same argument can be applied to $g$, and by union bound, the statement of the lemma holds with probability at least $1-\delta$.
\end{proof}

\begin{lemma}[Lemma C.5 of \cite{zhou2021nearly}]\label{lem: zhou lemma c.5}
On the good event $\calE$ (\Cref{def:good event}), for any $h\in[H]$, we have
    \begin{align*}
        &\sum_{k=1}^K \left(V_{k,h}^r(s_h^k) - V_{h}^{r^k,\pi^k}(s_h^k)\right)\leq  2\widehat\beta_K \sqrt{\sum_{k=1}^K\sum_{j=1}^H (\bar\sigma_{k,j}^r)^2}\sqrt{2Hd\log(1 + K/\lambda)}+8H \sqrt{2T\log(2H/\delta)},\\
        &\sum_{k=1}^K \left(V_{k,h}^g(s_h^k) - V_{h}^{g,\pi^k}(s_h^k)\right)\leq  2\widehat\beta_K \sqrt{\sum_{k=1}^K\sum_{j=1}^H (\bar\sigma_{k,j}^g)^2}\sqrt{2Hd\log(1 + K/\lambda)}+8H \sqrt{2T\log(2H/\delta)}.
    \end{align*}
\end{lemma}
\begin{proof}
Consider the case of $r$. Note that
\begin{align*}
    &V_{k,h}^r(s_h^k) - V_h^{r^k,\pi^k}(s_h^k) \\
    &= V_{k,h}^r(s_h^k) - V_h^{r^k,\pi^k}(s_h^k) - Q_{k,h}^r(s_h^k,a_h^k) + Q_h^{r^k,\pi^k}(s_h^k,a_h^k) + Q_{k,h}^r(s_h^k,a_h^k) - Q_h^{r^k,\pi^k}(s_h^k,a_h^k).
\end{align*}
Note that $Q_{k,h}^r(s_h^k,a_h^k) - Q_h^{r^k,\pi^k}(s_h^k,a_h^k)$ can be further derived as follows.
\begin{align*}
    Q_{k,h}^r(s_h^k,a_h^k) - Q_h^{r^k,\pi^k}(s_h^k,a_h^k)
    &\leq \langle \phi_{V_{k,h+1}^r}(s_h^k,a_h^k),\widehat\theta_{k,h}^r\rangle + \widehat\beta_k\left\|\left(\widehat\Sigma_{k,h}^r\right)^{-1/2}\phi_{V_{k,h+1}^r}(s_h^k,a_h^k)\right\|_2 -  P _h V_{h+1}^{r^k,\pi^k}(s_h^k,a_h^k)\\
    &= \langle \phi_{V_{k,h+1}^r}(s_h^k,a_h^k),\widehat\theta_{k,h}^r - \theta_h^*\rangle + \widehat\beta_k\left\|\left(\widehat\Sigma_{k,h}^r\right)^{-1/2}\phi_{V_{k,h+1}^r}(s_h^k,a_h^k)\right\|_2 +  P _h \left(V_{k,h+1}^r - V_{h+1}^{r^k,\pi^k}\right)(s_h^k,a_h^k)\\
    &\leq \left\|\left(\widehat\Sigma_{k,h}^r\right)^{-1/2} \phi_{V_{k,h+1}^r}(s_h^k,a_h^k)\right\|_2\left\|\left(\widehat\Sigma_{k,h}^r\right)^{1/2}(\widehat\theta_{k,h}^r - \theta_h^*)\right\|_2 + \widehat\beta_k\left\|\left(\widehat\Sigma_{k,h}^r\right)^{-1/2}\phi_{V_{k,h+1}^r}(s_h^k,a_h^k)\right\|_2 \\
    &\quad+  P _h \left(V_{k,h+1}^r - V_{h+1}^{r^k,\pi^k}\right)(s_h^k,a_h^k)\\
    &\leq 2\widehat\beta_k\left\|\left(\widehat\Sigma_{k,h}^r\right)^{-1/2}\phi_{V_{k,h+1}^r}(s_h^k,a_h^k)\right\|_2 +  P _h \left(V_{k,h+1}^r - V_{h+1}^{r^k,\pi^k}\right)(s_h^k,a_h^k)
\end{align*}
where the first inequality is due to the definition of $Q_{k,h}^r, Q_{h}^{r^k,\pi^k}$ and that $Q_{k,h}^r(s,a) \geq Q_{h}^{r^k,\pi^k}(s,a) \geq 0$ (\Cref{lem:optimism 1}), the second inequality is due to the Cauchy-Schwarz inequality, and the last inequality is due to the good event $\calE$. Note that $Q_{k,h}^r(s_h^k) - Q_h^{r^k,\pi^k}(s_h^k) \leq H$. Then we take $\min\{H, \cdot\}$ on both sides.
\begin{align*}
    Q_{k,h}^r(s_h^k, a_h^k) - Q_h^{r^k,\pi^k}(s_h^k,a_h^k) &\leq \min\left\{H, 2\widehat\beta_k\left\|\left(\widehat\Sigma_{k,h}^r\right)^{-1/2}\phi_{V_{k,h+1}^r}(s_h^k,a_h^k)\right\|_2 +  P _h \left(V_{k,h+1}^r - V_{h+1}^{r^k,\pi^k}\right)(s_h^k,a_h^k)\right\}\\
    &\leq \min\left\{H, 2\widehat\beta_k\left\|\left(\widehat\Sigma_{k,h}^r\right)^{-1/2}\phi_{V_{k,h+1}^r}(s_h^k,a_h^k)\right\|_2\right\} +  P _h \left(V_{k,h+1}^r - V_{h+1}^{r^k,\pi^k}\right)(s_h^k,a_h^k) \\
    &= \min\left\{H, 2\widehat\beta_k\left\|\left(\widehat\Sigma_{k,h}^r\right)^{-1/2}\phi_{V_{k,h+1}^r}(s_h^k,a_h^k)\right\|_2\right\} \\
    &\quad+  P _h \left(V_{k,h+1}^r - V_{h+1}^{r^k,\pi^k}\right)(s_h^k,a_h^k) - \left(V_{k,h+1}^r -V_{h+1}^{r^k,\pi^k}\right)(s_{h+1}^k) \\
    &\quad+ \left(V_{k,h+1}^r -V_{h+1}^{r^k,\pi^k}\right)(s_{h+1}^k)
\end{align*}
where the second inequality is true because $ P _h \left(V_{k,h+1}^r - V_{h+1}^{r^k,\pi^k}\right)(s_h^k,a_h^k)\geq 0$ due to \Cref{lem: beta hat}. 
We deduce that
\begin{align*}
    V_{k,h}^r(s_h^k) - V_h^{r^k,\pi^k}(s_h^k)
    &\leq V_{k,h}^r(s_h^k) - V_h^{r^k,\pi^k}(s_h^k) - Q_{k,h}^r(s_h^k,a_h^k) + Q_h^{r^k,\pi^k}(s_h^k,a_h^k) + \min\left\{H, 2\widehat\beta_k\left\|\left(\widehat\Sigma_{k,h}^r\right)^{-1/2}\phi_{V_{k,h+1}^r}(s_h^k,a_h^k)\right\|_2\right\} \\
    &\quad+P _h \left(V_{k,h+1}^r - V_{h+1}^{r^k,\pi^k}\right)(s_h^k,a_h^k) - \left(V_{k,h+1}^r -V_{h+1}^{r^k,\pi^k}\right)(s_{h+1}^k) + V_{k,h+1}^r(s_{h+1}^k) -V_{h+1}^{r^k,\pi^k}(s_{h+1}^k)
\end{align*}
Due to the above recursion, it follows that
\begin{align*}
    V_{k,h}^r(s_h^k) - V_h^{r^k,\pi^k}(s_h^k)
    &\leq \sum_{j=h}^H V_{k,j}^r(s_j^k) - V_j^{r^k,\pi^k}(s_j^k) - Q_{k,j}^r(s_j^k,a_j^k) + Q_j^{r^k,\pi^k}(s_j^k,a_j^k)\\
    &\quad + \sum_{j=h}^H P_j \left(V_{k,j+1}^r - V_{j+1}^{r^k,\pi^k}\right)(s_j^k,a_j^k) - \left(V_{k,j+1}^r -V_{j+1}^{r^k,\pi^k}\right)(s_{j+1}^k)\\
    &\quad + \sum_{j=h}^H \min\left\{H, 2\widehat\beta_k\left\|\left(\widehat\Sigma_{k,j}^r\right)^{-1/2}\phi_{V_{k,j+1}^r}(s_j^k,a_j^k)\right\|_2\right\}.
\end{align*}
By summing over $k=1,\ldots,K$, for all $h \in [H]$, we have
\begin{align*}
    \sum_{k=1}^K V_{k,h}^r(s_h^k) - V_h^{r^k,\pi^k}(s_h^k)
    &\leq \underbrace{\sum_{k=1}^K\sum_{j=h}^H V_{k,j}^r(s_j^k) - V_j^{r^k,\pi^k}(s_j^k) - Q_{k,j}^r(s_j^k,a_j^k) + Q_j^{r^k,\pi^k}(s_j^k,a_j^k)}_{\textnormal{(a)}}\\
    &\quad + \underbrace{\sum_{k=1}^K\sum_{j=h}^H P_j \left(V_{k,j+1}^r - V_{j+1}^{r^k,\pi^k}\right)(s_j^k,a_j^k) - \left(V_{k,j+1}^r -V_{j+1}^{r^k,\pi^k}\right)(s_{j+1}^k)}_{\textnormal{(b)}}\\
    &\quad + \underbrace{\sum_{k=1}^K\sum_{j=h}^H \min\left\{H, 2\widehat\beta_k\left\|\left(\widehat\Sigma_{k,j}^r\right)^{-1/2}\phi_{V_{k,j+1}^r}(s_j^k,a_j^k)\right\|_2\right\}}_{\textnormal{(c)}}.
\end{align*}
Now, we bound terms individually. On the good event $\calE$, (a) and (b) can be bounded as $4H\sqrt{2T\log(2H/\delta)}$. For (c),  we bound it as follows.
\begin{align*}
    \text{(c)} &\leq \sum_{k=1}^K\sum_{j=1}^H 2\widehat\beta_k \bar\sigma_{k,j}^r \min\left\{\frac{H}{2\widehat\beta_k \bar\sigma_{k,j}^r}, \left\|\left(\widehat\Sigma_{k,j}^r\right)^{-1/2}\phi_{V_{k,j+1}^r}(s_j^k,a_j^k)/\bar\sigma_{k,j}^r\right\|_2\right\}\\
    &\leq \sum_{k=1}^K\sum_{j=1}^H 2\widehat\beta_k \bar\sigma_{k,j}^r \min\left\{1, \left\|\left(\widehat\Sigma_{k,j}^r\right)^{-1/2}\phi_{V_{k,j+1}^r}(s_j^k,a_j^k)/\bar\sigma_{k,j}^r\right\|_2\right\}\\
    &\leq 2\widehat\beta_K \sum_{k=1}^K\sum_{j=1}^H \bar\sigma_{k,j}^r \min\left\{1, \left\|\left(\widehat\Sigma_{k,j}^r\right)^{-1/2}\phi_{V_{k,j+1}^r}(s_j^k,a_j^k)/\bar\sigma_{k,j}^r\right\|_2\right\}\\
    &\leq 2\widehat\beta_K \sqrt{\sum_{k=1}^K\sum_{j=1}^H (\bar\sigma_{k,j}^r)^2}\sqrt{\sum_{k=1}^K\sum_{j=1}^H \min\left\{1, \left\|\left(\widehat\Sigma_{k,j}^r\right)^{-1/2}\phi_{V_{k,j+1}^r}(s_j^k,a_j^k)/\bar\sigma_{k,j}^r\right\|_2^2\right\}}
\end{align*}
where the second inequality is due to ${H}/(2\widehat\beta_k \bar\sigma_{k,j}^r)\leq 1$, and the third inequality is due to the fact that $\widehat\beta_K \geq \widehat\beta_k$ for all $k\in[K]$. This is because ${H}/(2\widehat\beta_k \bar\sigma_{k,j}^r) \leq H / (\sqrt{d} H/\sqrt{d}) = 1$, due to $\widehat\beta_k \geq \sqrt{d}$, $\bar\sigma_{k,j}^r \geq H/\sqrt{d}.$ Furthermore, the last inequality is due to the Cauchy-Schwarz inequality. Finally, by \Cref{lem:elliptical}, term (c) is bounded as
\begin{align*}
    \text{(c)} \leq 2\widehat\beta_K \sqrt{\sum_{k=1}^K\sum_{j=1}^H (\bar\sigma_{k,j}^r)^2}\sqrt{2Hd\log(1 + K/\lambda)}.
\end{align*}
Consequently, we deduce that for all $h\in [H]$,
\[
    \sum_{k=1}^K V_{k,h}^r(s_h^k) - V_h^{r^k,\pi^k}(s_h^k) \leq 2\widehat\beta_K \sqrt{\sum_{k=1}^K\sum_{j=1}^H (\bar\sigma_{k,j}^r)^2}\sqrt{2Hd\log(1 + K/\lambda)}+8H \sqrt{2T\log(2H/\delta)}.
\]
We conclude the proof by applying the same argument to $g$.
\end{proof}

\begin{lemma}[Lemma C.6 of \cite{zhou2021nearly}]\label{lem: zhou lemma c.6}
    Let $\lambda = 1/B^2$. On the good event $\calE$ (\Cref{def:good event}), for any $\ell \in \{r,g\}$
    \begin{align*}
        \sum_{k=1}^K\sum_{h=1}^H (\bar\sigma_{k,h}^\ell)^2 = \bigO\left(\frac{H^2T}{d} + HT + H^5d^2 + H^5 + H^4d^3\right).
    \end{align*}
\end{lemma}
\begin{proof}
    Consider the case of $r$. By the definition of $\bar\sigma_{k,h}^r$,
    \begin{align*}
        \sum_{k=1}^K\sum_{h=1}^H (\bar\sigma_{k,h}^r)^2 &\leq \sum_{k=1}^K\sum_{h=1}^H \frac{H^2}{d} + \sum_{k=1}^K\sum_{h=1}^H \bar\bbV_h V_{k,h+1}^r(s_h^k,a_h^k) + \sum_{k=1}^K\sum_{h=1}^H E_{k,h}^r\\
        &=\frac{H^2 T}{d} + \underbrace{\sum_{k=1}^K\sum_{h=1}^H \left(\bar\bbV_h V_{k,h+1}^r(s_h^k,a_h^k) - \bbV_h V_{k,h+1}^r(s_h^k,a_h^k) - E_{k,h}^r\right)}_{\text{(a)}}\\
        &\quad+\underbrace{\sum_{k=1}^K\sum_{h=1}^H \bbV_h V_{h+1}^{r^k,\pi^k}(s_h^k,a_h^k)}_{\text{(b)}} + \underbrace{\sum_{k=1}^K\sum_{h=1}^H \left(\bbV_h V_{k,h+1}^r(s_h^k,a_h^k)-\bbV_h V_{h+1}^{r^k,\pi^k}(s_h^k,a_h^k)\right)}_{\text{(c)}}\\
        &\quad+ 2\underbrace{\sum_{k=1}^K\sum_{h=1}^H E_{k,h}^r}_{\text{(d)}}.
    \end{align*}
    Note that term (a) is nonpositive due to \Cref{lem: beta hat}, and \Cref{lem:ltv} implies that term (b) is bounded as
    \[
    \text{(b)} \leq 3(HT+H^3\log(2/\delta)).
    \]
    Next, we bound term (c) as follows. 
    \begin{align*}
        \text{(c)} &= \sum_{k=1}^K\sum_{h=1}^H \left( \langle\phi_{(V_{k,h+1}^r)^2} (s_h^k,a_h^k), \theta_h^*\rangle- \langle \phi_{V_{k,h+1}^r}(s_h^k,a_h^k), \theta_h^*\rangle^2\right)\\
        &\quad- \sum_{k=1}^K\sum_{h=1}^H\left(\langle\phi_{(V_{h+1}^{r,\pi^k})^2}(s_h^k,a_h^k), \theta_h^* \rangle - \langle \phi_{V_{k,h+1}^{r,\pi^k}}(s_h^k,a_h^k), \theta_h^*\rangle^2\right)\\
        &\leq \sum_{k=1}^K\sum_{h=1}^H\left(
        \langle\phi_{(V_{k,h+1}^r)^2} (s_h^k,a_h^k),\theta_h^*\rangle - \langle\phi_{(V_{h+1}^{r,\pi^k})^2}(s_h^k,a_h^k), \theta_h^* \rangle
        \right)\\
        &=\sum_{k=1}^K\sum_{h=1}^H  P _h\left(\left(V_{k,h+1}^r\right)^2 - \left(V_{h+1}^{r,\pi^k}\right)^2\right)(s_h^k,a_h^k)\\
        &\leq 2H\sum_{k=1}^K\sum_{h=1}^H  P _h\left(V_{k,h+1}^r - V_{h+1}^{r,\pi^k}\right)(s_h^k,a_h^k)
    \end{align*}
    where the first inequality is due to $ P _h V_{k,h+1}^r(s,a) \geq  P _h V_{h+1}^{r,\pi^k}(s,a)$ by \Cref{lem:optimism 1}, and the last inequality is due to $V_{k,h+1}^r(s,a) + V_{h+1}^{r,\pi^k}(s,a)\leq 2H$. We can further deduce as follows.
    \begin{align*}
        &\sum_{k=1}^K\sum_{h=1}^H P _h\left(V_{k,h+1}^r - V_{h+1}^{r,\pi^k}\right)(s_h^k,a_h^k) \\
        &=\sum_{k=1}^K\sum_{h=1}^H\left( P _h\left(V_{k,h+1}^r - V_{h+1}^{r,\pi^k}\right)(s_h^k,a_h^k)-\left(V_{k,h+1}^r - V_{h+1}^{r,\pi^k}\right)(s_{h+1}^k)\right)+\sum_{k=1}^K\sum_{h=1}^H \left(V_{k,h+1}^r - V_{h+1}^{r,\pi^k}\right)(s_{h+1}^k) \\
        &\leq 4H\sqrt{2T\log(2H/\delta)} + (H-1)\left(2\widehat\beta_K \sqrt{\sum_{k=1}^K\sum_{h=1}^H (\bar\sigma_{k,h}^r)^2}\sqrt{2Hd\log(1 + K/\lambda)}+8H\sqrt{2T\log(2H/\delta)}\right)\\
        &\leq 2H\widehat\beta_K \sqrt{\sum_{k=1}^K\sum_{h=1}^H (\bar\sigma_{k,h}^r)^2}\sqrt{2Hd\log(1 + K/\lambda)} + 12H^2\sqrt{2T\log(2H/\delta)}
    \end{align*}
    where the first inequality is due to Lemmas \ref{lem:mds term} and \ref{lem: zhou lemma c.5}. Then, term (c) is bounded as
    \begin{align*}
        \text{(c)}\leq 4H^2\widehat\beta_K \sqrt{\sum_{k=1}^K\sum_{h=1}^H (\bar\sigma_{k,h}^r)^2}\sqrt{2Hd\log(1 + K/\lambda)} + 24H^3\sqrt{2T\log(2H/\delta)}.
    \end{align*}

    Next, term (d) is bounded as follows.
    \begin{align*}
        \sum_{k=1}^K\sum_{h=1}^H E_{k,h}^r &= \sum_{k=1}^K\sum_{h=1}^H \min\left\{H^2, \widetilde\beta_k\left\|(\widetilde\Sigma_{k,h}^r)^{-1/2}\phi_{(V_{k,h+1}^r)^2}(s_h^k,a_h^k)\right\|_2\right\}\\
        &\quad+\sum_{k=1}^K\sum_{h=1}^H \min\left\{H^2, 2H\check\beta_k \left\|(\widehat\Sigma_{k,h}^r)^{-1/2}\phi_{V_{k,h+1}^r}(s_h^k,a_h^k)\right\|_2\right\} \\
        &= \sum_{k=1}^K\sum_{h=1}^H \widetilde\beta_k\min\left\{\frac{H^2}{\widetilde\beta_k}, \left\|(\widetilde\Sigma_{k,h}^r)^{-1/2}\phi_{(V_{k,h+1}^r)^2}(s_h^k,a_h^k)\right\|_2\right\}\\
        &\quad+\sum_{k=1}^K\sum_{h=1}^H 2H\check\beta_k\bar\sigma_{k,h}^r\min\left\{\frac{H}{2\check\beta_k \bar\sigma_{k,h}^r}, \left\|(\widehat\Sigma_{k,h}^r)^{-1/2}\phi_{V_{k,h+1}^r}(s_h^k,a_h^k) / \bar\sigma_{k,h}^r\right\|_2\right\}.
    \end{align*}
    Note that $\widetilde\beta_K \geq \widetilde\beta_k$, $\check\beta_K \geq \check\beta_k$ for all $k\in[K]$. Furthermore, $\widetilde\beta_k \geq H^2$, $\check\beta_k \bar\sigma_{k,h}^r \geq d\sqrt{H^2/d}\geq H$, and $\bar\sigma_{k,h}^r$ is bounded as
    \[
    \bar\sigma_{k,h}^r \leq \sqrt{\max\left\{H^2/d, H^2 + 2H^2\right\}} \leq 2H.
    \]
    Then we have
    \begin{align*}
        \sum_{k=1}^K\sum_{h=1}^H E_{k,h}^r
        &\leq \widetilde\beta_K\sum_{k=1}^K\sum_{h=1}^H \min\left\{1, \left\|(\widetilde\Sigma_{k,h}^r)^{-1/2}\phi_{(V_{k,h+1}^r)^2}(s_h^k,a_h^k)\right\|_2\right\}\\
        &\quad+4H^2\check\beta_K\sum_{k=1}^K\sum_{h=1}^H \min\left\{1, \left\|(\widehat\Sigma_{k,h}^r)^{-1/2}\phi_{V_{k,h+1}^r}(s_h^k,a_h^k) / \bar\sigma_{k,h}^r\right\|_2\right\}\\
        &\leq \widetilde\beta_K\sqrt{T}\sqrt{\sum_{k=1}^K\sum_{h=1}^H \min\left\{1, \left\|(\widetilde\Sigma_{k,h}^r)^{-1/2}\phi_{(V_{k,h+1}^r)^2}(s_h^k,a_h^k)\right\|_2^2\right\}}\\
        &\quad+4H^2\check\beta_K\sqrt{T}\sqrt{\sum_{k=1}^K\sum_{h=1}^H \min\left\{1, \left\|(\widehat\Sigma_{k,h}^r)^{-1/2}\phi_{V_{k,h+1}^r}(s_h^k,a_h^k) / \bar\sigma_{k,h}^r\right\|_2^2\right\}}\\
        &\leq \widetilde\beta_K\sqrt{T}\sqrt{2Hd\log(1+H^4K/(d\lambda))} + 4H^2\check\beta_K \sqrt{T}\sqrt{2Hd\log(1+K/\lambda)}
    \end{align*}
    where the last inequality follows from \Cref{lem:elliptical}. Finally, with all things together,
    \begin{align*}
        \sum_{k=1}^K\sum_{h=1}^H (\bar\sigma_{k,h}^r)^2 &\leq \frac{H^2 T}{d} + 3(HT+H^3\log(2/\delta)) \\
        &\quad+4H^2\widehat\beta_K \sqrt{\sum_{k=1}^K\sum_{h=1}^H (\bar\sigma_{k,h}^r)^2}\sqrt{2Hd\log(1 + K/\lambda)} + 16H^3\sqrt{2T\log(2H/\delta)} \\
        &\quad+2\widetilde\beta_K\sqrt{T}\sqrt{2Hd\log(1+H^4K/(d\lambda))} + 8H^2\check\beta_K \sqrt{T}\sqrt{2Hd\log(1+K/\lambda)}.
    \end{align*}
    For $\lambda=1/B^2$, we have $\widehat\beta_K = \bigO(\sqrt{d})$, $\widetilde\beta_K = \bigO(H^2\sqrt{d})$, and $\check\beta_K=\bigO\left(d\right)$. Then it can be rewritten as
    \begin{align*}
        \sum_{k=1}^K\sum_{h=1}^H (\bar\sigma_{k,h}^r)^2
        &= \bigO\left(\frac{H^2T}{d} + HT + H^3 + H^2 \sqrt{d}\sqrt{\sum_{k=1}^K\sum_{h=1}^H (\bar\sigma_{k,h}^r)^2}\sqrt{Hd}+H^3\sqrt{T}+H^2\sqrt{T}\sqrt{Hd} + H^2d\sqrt{T}\sqrt{Hd}\right)\\
        &= \bigO\left(\frac{H^2T}{d} + HT + H^{2.5} d\sqrt{\sum_{k=1}^K\sum_{h=1}^H (\bar\sigma_{k,h}^r)^2}+H^3\sqrt{T}+H^{2.5}d^{0.5}\sqrt{T} + H^{2.5}d^{1.5}\sqrt{T}\right).
    \end{align*}
    Due to the AM-GM inequality, we have
    \begin{align*}
    H^3\sqrt{T} &= \bigO\left(HT + H^5\right), \\
    H^{2.5}d^{0.5}\sqrt{T} &= \bigO\left(HT + H^4d\right),\\
    H^{2.5}d^{1.5}\sqrt{T} &= \bigO\left(HT+H^4d^3\right).
    \end{align*}
    Then it follows that
    \begin{align*}
        &\sum_{k=1}^K\sum_{h=1}^H (\bar\sigma_{k,h}^r)^2 =\bigO\left(\frac{H^2T}{d} + HT + H^{2.5} d\xi\sqrt{\sum_{k=1}^K\sum_{h=1}^H (\bar\sigma_{k,h}^r)^2}+H^5 + H^4d^3\right).
    \end{align*}
    Furthermore, we know that if $x \leq a\sqrt{x} + b$, then $x\leq (3/2)(a^2+b)$. Finally, we have
    \[
        \sum_{k=1}^K\sum_{h=1}^H (\bar\sigma_{k,h}^r)^2 = \bigO\left(\frac{H^2T}{d} + HT + H^5d^2 + H^5 + H^4d^3\right).
    \]
    We conclude the proof by applying the same argument to $g$.
\end{proof}

\begin{proof}[Proof of \Cref{lem:bias}]
On the good event $\calE$ (\Cref{def:good event}), by Lemmas \ref{lem: zhou lemma c.5} and \ref{lem: zhou lemma c.6}, we have
    \begin{align*}
        &\sum_{k=1}^K \left(V_{k,h}^r(s_h^k) - V_{h}^{r^k,\pi^k}(s_h^k)\right)=\bigO\left(\sqrt{dH^4 K} + \sqrt{d^2 H^3 K}+ d^{2.5}H^3\right) ,\\
        &\sum_{k=1}^K \left(V_{k,h}^g(s_h^k) - V_{h}^{g,\pi^k}(s_h^k)\right) = \bigO\left(\sqrt{dH^4 K} + \sqrt{d^2 H^3 K}+ d^{2.5}H^3\right).
    \end{align*}
    Moreover, under $K=\Omega(d^3 H^3)$, we have $d^{2.5}H^3 = \bigO(\sqrt{d^2 H^ 3 K})$. This concludes the proof.
\end{proof}

\section{Proof of Lemma~\ref{lem:Yk bound}}

In this section, we provide more detailed proof of \Cref{lem:Yk bound}.
\begin{lemma}\label{lem: Yk naive bound}
    Suppose that $\eta \leq 1$, $\alpha \leq 1/H^2$, and $\theta \leq 1/(2H)$. For all $k\in [K]$, $Y_k \leq 3H\eta k$.
\end{lemma}
\begin{proof}
    Note that our dual update can be written as $Y_{k+1} = \left[(1-\alpha\eta H^3)Y_k + \eta\left(b-V_{k,1}^g(s_1) - \alpha H^3 -2\theta H^2\right)\right]_+$ for $k \geq 1$. It follows that
    \begin{align*}
        Y_{k+1} 
        &\leq |(1-\alpha \eta H^3)Y_k| + \eta\left|b-V_{k,h}^g(s_1) - \alpha H^3 - 2\theta H^2\right|\\
        &\leq Y_k + \eta\left|b-V_{k,h}^g(s_1) - \alpha H^3 - 2\theta H^2\right|.
    \end{align*}
     where the second inequality follows from $0\leq(1-\alpha \eta H^3)Y_k \leq Y_k$ by the assumption. Moreover, $|b - V_{k,h}^g(s_1)| \leq H$, $\alpha H^3 \leq H$, and $2\theta H^2 \leq H$. Then by the triangle inequality for all $k$,
    \begin{align*}
        Y_{k+1} \leq Y_k + 3\eta H.
    \end{align*}
    Recall that $Y_1 = 0$. Then we have $Y_{k+1} \leq 3\eta H k$. This concludes the proof.
\end{proof}

\begin{lemma}\label{lem:drift}
    Suppose that $\eta \leq 1$, $\alpha \leq 1/H^2$, and $\theta \leq 1/(2H)$. On the good event $\calE$ (\Cref{def:good event}), for all $k\in [K]$
    \begin{align}\label{eq:drift}
        &\frac{Y_{k+1}^2 - Y_{k}^2}{2} \leq -\gamma \eta Y_k + \frac{\eta}{\alpha} \bbE_{\bar\pi}\left[\sum_{h=1}^H D(\bar\pi_h(\cdot|s_h)||\widetilde\pi_h^k(\cdot|s_h)) - D(\bar\pi_h(\cdot|s_h)|| \pi_h^{k+1}(\cdot|s_h))\right] + C
    \end{align}
    where $C$ is defined as
    \begin{align*}
         C =  \frac{\alpha \eta H^3}{2} + 2\eta H^2\theta + 2\eta H^2 + 2\eta^2(H^2 + \alpha^2 H^6 + 9 \eta^2 \alpha^2 H^8K^2 + 4\theta^2 H^4).
    \end{align*}
\end{lemma}
\begin{proof}
    The proof closely follows the proof of Lemma 17 in \cite{yu2026primaldual}. By the definition of $Y_k$, we have
    \begin{align*}
        Y_{k+1}^2 \leq Y_k^2 + 2 Y_k \eta\left( b- V_{k,1}^g(s_1) - \alpha H^3(1+Y_k) - 2\theta H^2\right) + \eta^2\left( b- V_{k,1}^g(s_1) - \alpha H^3(1+Y_k) - 2\theta H^2\right)^2.
    \end{align*}
    We rearrange it as follows.
    \begin{align*}
        \frac{Y_{k+1}^2 - Y_{k}^2}{2}
        &\leq \underbrace{Y_{k}\eta\left( b- V_{k,1}^g(s_1) -\alpha H^3(1+Y_k) - 2\theta H^2\right)}_{\text{(I)}} + \underbrace{\frac{\eta^2}{2}\left(b- V_{k,1}^g(s_1) - \alpha H^3(1+Y_k) - 2\theta H^2\right)^2}_{\text{(II)}}.
    \end{align*}
    Term (II) is bounded as
    \begin{align*}
        \text{(II)} 
        &\leq 2\eta^2((b- V_{k,1}^g(s_1))^2 + \alpha^2 H^6 + \alpha^2 H^6Y_k^2 + 4\theta^2 H^4) \\
        &\leq 2\eta^2(H^2 + \alpha^2 H^6 + 9 \eta^2 \alpha^2 H^8K^2 + 4\theta^2 H^4)\\
    \end{align*}
    where the first inequality follows from the Cauchy-Schwarz inequality, the second inequality follows from $|b- V_{k,1}^g(s_1)| \leq H$ and $0\leq Y_k \leq 3\eta HK$ for all $k$ (\Cref{lem: Yk naive bound}).

    Now, we further deduce (I). Recall that $\bar\pi$ is the Slater policy that satisfies $V_1^{g,\bar\pi}(s_1) \geq b+ \gamma$ for some $\gamma > 0$. We have
    \begin{align*}
        V_1^{g,\bar\pi}(s_1) - V_{k,1}^g(s_1)
        &=\bbE_{\bar\pi} \left[ \sum_{h=1}^H \langle Q_{k,h}^g(s_h, \cdot), \bar\pi_h(\cdot|s_h) - \pi_h^k(\cdot|s_h) \rangle |s_1\right] \\
        &\quad+\bbE_{\bar\pi} \left[ \sum_{h=1}^H g_h(s_h,a_h) + P_h V_{h+1}^g(s_h,a_h) - Q_{k,h}^g(s_h,a_h)|s_1 \right] \\
        &\leq \bbE_{\bar\pi} \left[ \sum_{h=1}^H \langle Q_{k,h}^g(s_h, \cdot), \bar\pi_h(\cdot|s_h) - \pi_h^k(\cdot|s_h) \rangle |s_1\right]
    \end{align*}
     where the equality is due to \Cref{lem:extended value diff}, and the inequality is due to \Cref{lem:optimism 1}. Due to the Slater assumption, the above can be written as
     \begin{align*}
         b - V_{k,1}^g(s_1) \leq \bbE_{\bar\pi} \left[ \sum_{h=1}^H \langle Q_{k,h}^g(s_h, \cdot), \bar\pi_h(\cdot|s_h) - \pi_h^k(\cdot|s_h) \rangle |s_1\right] - \gamma.
     \end{align*}
     By adopting this into (I), it follows that
     \begin{align*}
         \textnormal{(I)} \leq -\gamma \eta Y_k  + \underbrace{Y_k \eta\left(\bbE_{\bar\pi} \left[ \sum_{h=1}^H \langle Q_{k,h}^g(s_h, \cdot), \bar\pi_h(\cdot|s_h) - \pi_h^k(\cdot|s_h) \rangle |s_1\right]-\alpha H^3(1+Y_k) - 2\theta H^2\right)}_{\textnormal{(III)}}.
     \end{align*}
     We further deduce (III) as follows. Note that $\pi_h^{k+1}(\cdot|s) \in \argmax_{\pi} \langle Q_{k,h}^r(s,\cdot) + Y_k Q_{k,h}^g(s,\cdot), \pi(\cdot|s) \rangle - \frac{1}{\alpha} D(\pi(\cdot|s)|| \widetilde\pi_h^k(\cdot|s))$. Then by \Cref{lem:pushback}, 
     \begin{align*}
         &\langle Q_{k,h}^r(s_h,\cdot) + Y_k Q_{k,h}^g(s_h,\cdot), \pi_h^{k+1}(\cdot|s_h) \rangle - \frac{1}{\alpha}D(\pi_h^{k+1}(\cdot|s_h)|| \widetilde\pi_h^k(\cdot|s_h)) \\
         &\geq \langle Q_{k,h}^r(s_h,\cdot) + Y_k Q_{k,h}^g(s_h,\cdot), \bar\pi_h(\cdot|s_h) \rangle -\frac{1}{\alpha}D(\bar\pi_h(\cdot|s_h)||\widetilde\pi_h^k(\cdot|s_h)) + \frac{1}{\alpha}D(\bar\pi_h(\cdot|s_h)|| \pi_h^{k+1}(\cdot|s_h)).
     \end{align*}
    It can be rewritten as
    \begin{align*}
        Y_k \langle Q_{k,h}^g(s_h,\cdot), \bar\pi_h(\cdot|s_h)- \pi_h^{k}(\cdot|s_h)\rangle
        &\leq \langle Q_{k,h}^r(s_h,\cdot), \pi_h^{k+1}(\cdot|s_h)- \pi_h^{k}(\cdot|s_h)\rangle -\frac{1}{\alpha}D(\pi_h^{k+1}(\cdot|s_h)|| \widetilde\pi_h^k(\cdot|s_h))\\
        &\quad+ \langle Q_{k,h}^r(s_h,\cdot), \pi_h^{k}(\cdot|s_h) - \bar\pi_h(\cdot|s_h)\rangle +\frac{1}{\alpha}D(\bar\pi_h(\cdot|s_h)||\widetilde\pi_h^k(\cdot|s_h)) - \frac{1}{\alpha}D(\bar\pi_h(\cdot|s_h)|| \pi_h^{k+1}(\cdot|s_h)) \\
        &\quad+ Y_k \langle Q_{k,h}^g(s_h,\cdot), \pi_h^{k+1}(\cdot|s_h) - \pi_h^k(\cdot|s_h) \rangle \\
        &\leq \frac{\alpha H^2}{2} + 2H\theta + 2H + \frac{1}{\alpha}D(\bar\pi_h(\cdot|s_h)||\widetilde\pi_h^k(\cdot|s_h)) - \frac{1}{\alpha}D(\bar\pi_h(\cdot|s_h)|| \pi_h^{k+1}(\cdot|s_h)) \\
        &\quad+ Y_k \left( \alpha H^2 (1+Y_k) + 2\theta H\right)
    \end{align*}
     where the second inequality comes from \Cref{lem:KL mixing} and that $\langle Q_{k,h}^r(s_h,\cdot), \pi_h^{k+1}(\cdot|s_h)- \bar\pi_h(\cdot|s_h)\rangle \leq \| Q_{k,h}^r(s_h,\cdot)\|_\infty \|\pi_h^{k+1}(\cdot|s_h)- \bar\pi_h(\cdot|s_h)\|_1 \leq 2H$ (H\"older's inequality).
     Now, we take the sum over $h=1,\ldots,H$ and $\bbE_{\bar\pi}$ on both sides. Moreover, we multiply both sides by $\eta$. Then, it is written as
     \begin{align*}
         &Y_k \eta\left(\bbE_{\bar\pi}\left[\sum_{h=1}^H\langle Q_{k,h}^g(s_h,\cdot), \bar\pi_h(\cdot|s_h)- \pi_h^{k}(\cdot|s_h)\rangle\right] - \alpha H^3 (1+Y_k) - 2\theta H^2\right) \\
         &\leq \frac{\alpha \eta H^3}{2} + 2\eta H^2\theta + 2\eta H^2 + \frac{\eta}{\alpha} \bbE_{\bar\pi}\left[\sum_{h=1}^H D(\bar\pi_h(\cdot|s_h)||\widetilde\pi_h^k(\cdot|s_h)) - D(\bar\pi_h(\cdot|s_h)|| \pi_h^{k+1}(\cdot|s_h))\right] 
     \end{align*}
     Observe that the left-hand side equals (III). Thus, it follows that
     \begin{align*}
         \textnormal{(I)} \leq -\gamma \eta Y_k + \frac{\alpha \eta H^3}{2} + 2\eta H^2\theta + 2\eta H^2 + \frac{\eta}{\alpha} \bbE_{\bar\pi}\left[\sum_{h=1}^H D(\bar\pi_h(\cdot|s_h)||\widetilde\pi_h^k(\cdot|s_h)) - D(\bar\pi_h(\cdot|s_h)|| \pi_h^{k+1}(\cdot|s_h))\right].
     \end{align*}
     Finally, we have
     \begin{align*}
        \frac{Y_{k+1}^2 - Y_{k}^2}{2} 
        &\leq -\gamma \eta Y_k + \frac{\alpha \eta H^3}{2} + 2\eta H^2\theta + 2\eta H^2 + \frac{\eta}{\alpha} \bbE_{\bar\pi}\left[\sum_{h=1}^H D(\bar\pi_h(\cdot|s_h)||\widetilde\pi_h^k(\cdot|s_h)) - D(\bar\pi_h(\cdot|s_h)|| \pi_h^{k+1}(\cdot|s_h))\right] \\
        &\quad+ 2\eta^2(H^2 + \alpha^2 H^6 + 9 \eta^2 \alpha^2 H^8K^2 + 4\theta^2 H^4).
    \end{align*}
\end{proof}

\begin{lemma} [Restatement of \Cref{lem:Yk bound}]
Let $K \geq \max\{2H, H^2\}$. Suppose that we set $\eta=1/(H\sqrt{K}),\ \alpha = 1/(H^2\sqrt{K}),\ \theta=1/K$ and $\calE$ (\Cref{def:good event}) holds. For all $k \in [K]$,
    \[
        Y_k = \bigO(H^2/\gamma).
    \]
\end{lemma}
\begin{proof}
    Under $K \geq 2H$, we know that $\eta\leq1, \ \alpha \leq 1/H^2, \theta \leq 1/(2H)$. Thus, Lemmas \ref{lem: Yk naive bound} and \ref{lem:drift} are applicable. Note that
    \begin{align}\label{eq:Yk naive absolute diff}
    \begin{aligned}
        |Y_{k+1} - Y_k| 
        &\leq |-\alpha \eta H^3 Y_k + \eta(b-V_{k,1}^g(s_1) - \alpha H^3 - 2\theta H^2)|\\
        &\leq 3\alpha \eta^2 H^4 k + 3\eta H\\
        &\leq 3\alpha \eta^2 H^4 K + 3\eta H
    \end{aligned}
    \end{align}
    where the first inequality follows from that $|\max\{0, x\} - y| \leq |x-y|$ for any $x\in\bbR$ and $y \in \bbR_+$, and the second inequality follows from \Cref{lem: Yk naive bound}. Moreover, by summing \eqref{eq:drift} over from $k$ to $k+n_0-1$ (later, $n_0$ will be specified), it follows that
    \begin{align*}
        \frac{Y_{k+n_0}^2 - Y_k^2}{2} \leq -\gamma \eta \sum_{\tau=k}^{k+n_0-1} Y_\tau + \frac{\eta}{\alpha} \underbrace{\sum_{\tau=k}^{k+n_0-1}\bbE_{\bar\pi}\left[\sum_{h=1}^H D(\bar\pi_h(\cdot|s_h)||\widetilde\pi_h^\tau(\cdot|s_h)) - D(\bar\pi_h(\cdot|s_h)|| \pi_h^{\tau+1}(\cdot|s_h))\right]}_{\textnormal{(I)}} + Cn_0.
    \end{align*}
    Term (I) can be bounded as
    \begin{align*}
        \textnormal{(I)} 
        &= \sum_{\tau=k}^{k+n_0-1}\bbE_{\bar\pi}\left[\sum_{h=1}^H D(\bar\pi_h(\cdot|s_h)||\widetilde\pi_h^\tau(\cdot|s_h)) - D(\bar\pi_h(\cdot|s_h)|| \widetilde\pi_h^{\tau+1}(\cdot|s_h))\right] \\
        &\quad+ \sum_{\tau=k}^{k+n_0-1}\bbE_{\bar\pi}\left[\sum_{h=1}^H D(\bar\pi_h(\cdot|s_h)||\widetilde\pi_h^{\tau+1}(\cdot|s_h)) - D(\bar\pi_h(\cdot|s_h)||\pi_h^{\tau+1}(\cdot|s_h))\right]\\
        &\leq \sum_{h=1}^H\bbE_{\bar\pi}\left[D(\bar\pi_h(\cdot|s_h)||\widetilde\pi_h^k(\cdot|s_h)) -  D(\bar\pi_h(\cdot|s_h)||\widetilde\pi_h^{k+n_0}(\cdot|s_h))\right] + n_0 H \theta\log|\calA|\\
        &\leq H\log(|\calA|/\theta) + n_0 H\theta\log|\calA| 
    \end{align*}
    where the second and third inequalities follows from \Cref{lem:KL mixing} and that the nonnegativeness of KL divergence. Therefore, we deduce that for any $n_0\geq1$ and $k \leq K-n_0$,
    \begin{align}\label{eq:drift n0}
        \frac{Y_{k+n_0}^2 - Y_k^2}{2} \leq -\gamma \eta \sum_{\tau=k}^{k+n_0-1} Y_\tau + C',
    \end{align}
    where $C' = 3\gamma \eta^2 H\frac{n_0(n_0-1)}{2} (\alpha \eta H^3K + 1) + \frac{\eta}{\alpha}(H\log(|\calA|/\theta) + n_0 H\theta\log|\calA|) + Cn_0$. 
    
    Suppose that there exists $k \in [K]$ such that $Y_k > \frac{2C'}{\eta \gamma n_0} + n_0(3\alpha \eta^2 H^4 K + 3\eta H)$. Then we can take $k_{\textnormal{hit}} = \min\{k\in [K]: Y_k > \frac{2C'}{\eta \gamma n_0} + n_0(3\alpha \eta^2 H^4 K + 3\eta H)\}$, i.e., the first episode such that $Y_k$ exceeds the threshold $\frac{2C'}{\eta \gamma n_0} + n_0(3\alpha \eta^2 H^4 K + 3\eta H)$. Since $|Y_{k+1} - Y_k|  \leq 3\alpha \eta^2 H^4 K + 3\eta H$, we know that $k_{\textnormal{hit}} > n_0$ and $Y_{k_{\textnormal{hit}}-n_0},\ldots, Y_{k_{\textnormal{hit}}-1} \geq \frac{2C'}{\eta \gamma n_0}$ (If not, $Y_{k_{\textnormal{hit}}}$ never reach the threshold). Then, by \eqref{eq:drift n0}, it implies that, 
    \begin{align}
        \frac{Y_{k_{\textnormal{hit}}}^2 - Y_{k_{\textnormal{hit}}-n_0}^2}{2} \leq -\gamma \eta \sum_{\tau=k_{\textnormal{hit}}-n_0}^{k_{\textnormal{hit}}} Y_\tau + C' \leq -C' < 0.
    \end{align}
    This implies that $Y_{k_{\textnormal{hit}}} < Y_{k_{\textnormal{hit}}-n_0} < \frac{2C'}{\eta \gamma n_0} + n_0(3\alpha \eta^2 H^4 K + 3\eta H)$, where the second inequality follows from the fact that $k_{\textnormal{hit}} -n_0 < k_{\textnormal{hit}}$ and $k_{\textnormal{hit}}$ is the first episode that exceeds the threshold. This contradicts $Y_{k_{\textnormal{hit}}} > \frac{2C'}{\eta \gamma n_0} + n_0(3\alpha \eta^2 H^4 K + 3\eta H)$. Therefore, $Y_k \leq \frac{2C'}{\eta \gamma n_0} + n_0(3\alpha \eta^2 H^4 K + 3\eta H)$ for all $k\in [K].$ By taking $n_0 = H\sqrt{K},\ \eta = \frac{1}{H\sqrt{K}}, \ \alpha = \frac{1}{H^2\sqrt{K}},\ \theta = \frac{1}{K},$ it follows that $Y_k = \bigO(H^2/\gamma).$
\end{proof}

\section{Proof of Lemma~\ref{lem:omd}}
\begin{proof}
    Fix $k \in [K]$. Note that
    \begin{align*}
        V_1^{r^k,\pi^*}(s_1) + Y_k V_1^{g,\pi^*}(s_1) - V_{k,1}^r(s_1) - Y_k V_{k,1}^g(s_1)
        &= \bbE_{\pi^*}\left[\sum_{h=1}^H \langle Q_{k,h}^r(s_h,\cdot) + Y_kQ_{k,h}^g(s_h,\cdot), \pi_h^*(\cdot|s_h) - \pi_h^k(\cdot|s_h) \rangle | s_1\right] \\
        &\quad+ \bbE_{\pi^*}\left[\sum_{h=1}^H r_h^k(s_h,a_h) + P_h V_{k,h+1}^r(s_h,a_h) - Q_{k,h}^r(s_h,a_h) | s_1\right]\\
        &\quad+ Y_k\bbE_{\pi^*}\left[\sum_{h=1}^H g_h(s_h,a_h) + P_h V_{k,h+1}^g(s_h,a_h) - Q_{k,h}^g(s_h,a_h) | s_1\right].
    \end{align*}
    Note that the second and third terms are nonpositive by \Cref{lem:optimism 1}. Thus, we focus on the first term. Since $\pi_h^{k+1}(\cdot|s) \in \argmax_{\pi} \langle Q_{k,h}^r(s,\cdot) + Y_k Q_{k,h}^g(s,\cdot), \pi(\cdot|s) \rangle -\frac{1}{\alpha} D(\pi(\cdot|s)||\widetilde\pi_h^k(\cdot|s))$, by \Cref{lem:pushback}, for any $s\in\calS$,
    \begin{align*}
        &\langle Q_{k,h}^r(s,\cdot) + Y_k Q_{k,h}^g(s,\cdot), \pi_h^{k+1}(\cdot|s)\rangle - \frac{1}{\alpha}D(\pi_h^{k+1}(\cdot|s)|| \widetilde\pi_h^k(\cdot|s))\\
        &\geq \langle Q_{k,h}^r(s,\cdot) + Y_k Q_{k,h}^g(s,\cdot), \pi_h^{*}(\cdot|s) \rangle - \frac{1}{\alpha}D(\pi_h^*(\cdot|s)|| \widetilde\pi_h^k(\cdot|s)) + \frac{1}{\alpha}D(\pi_h^*(\cdot|s)||\pi_h^{k+1}(\cdot|s)).
    \end{align*}
    By rearranging the inequality, we have 
    \begin{align*}
        \langle Q_{k,h}^r(s,\cdot) + Y_k Q_{k,h}^g(s,\cdot), \pi_h^{*}(\cdot|s) - \pi_h^{k}(\cdot|s) \rangle
        &\leq \langle Q_{k,h}^r(s,\cdot) + Y_k Q_{k,h}^g(s,\cdot), \pi_h^{k+1}(\cdot|s) - \pi_h^{k}(\cdot|s) \rangle - \frac{1}{\alpha}D(\pi_h^{k+1}(\cdot|s)|| \widetilde\pi_h^k(\cdot|s)) \\
        &\quad + \frac{1}{\alpha}D(\pi_h^*(\cdot|s)|| \widetilde\pi_h^k(\cdot|s)) - \frac{1}{\alpha}D(\pi_h^*(\cdot|s)||\pi_h^{k+1}(\cdot|s)) \\
        &= \langle Q_{k,h}^r(s,\cdot) + Y_k Q_{k,h}^g(s,\cdot), \pi_h^{k+1}(\cdot|s) - \pi_h^{k}(\cdot|s) \rangle - \frac{1}{\alpha}D(\pi_h^{k+1}(\cdot|s)|| \widetilde\pi_h^k(\cdot|s)) \\
        &\quad + \frac{1}{\alpha}D(\pi_h^*(\cdot|s)|| \pi_h^k(\cdot|s)) - \frac{1}{\alpha}D(\pi_h^*(\cdot|s)||\pi_h^{k+1}(\cdot|s)) \\
        &\quad+ \frac{1}{\alpha}D(\pi_h^*(\cdot|s)|| \widetilde\pi_h^k(\cdot|s)) - \frac{1}{\alpha}D(\pi_h^*(\cdot|s)|| \pi_h^k(\cdot|s)).
    \end{align*}
    Then we take the sum over $k=1,\ldots,K$. Then it follows that
    \begin{align*}
        &\sum_{k=1}^K\langle Q_{k,h}^r(s,\cdot) + Y_k Q_{k,h}^g(s,\cdot), \pi_h^{*}(\cdot|s) - \pi_h^{k}(\cdot|s) \rangle \\
        &\leq \sum_{k=1}^K\langle Q_{k,h}^r(s,\cdot) + Y_k Q_{k,h}^g(s,\cdot), \pi_h^{k+1}(\cdot|s) - \pi_h^{k}(\cdot|s) \rangle - \sum_{k=1}^K \frac{1}{\alpha}D(\pi_h^{k+1}(\cdot|s)|| \widetilde\pi_h^k(\cdot|s))\\
        &\quad + \frac{1}{\alpha}D(\pi_h^*(\cdot|s)|| \pi_h^1(\cdot|s)) - \frac{1}{\alpha}D(\pi_h^*(\cdot|s)||\pi_h^{K+1}(\cdot|s)) \\
        &\quad+ \frac{1}{\alpha} \sum_{k=1}^K \left(D(\pi_h^*(\cdot|s)|| \widetilde\pi_h^k(\cdot|s)) - D(\pi_h^*(\cdot|s)|| \pi_h^k(\cdot|s))\right)\\
        &\leq \sum_{k=1}^K\langle Q_{k,h}^r(s,\cdot) + Y_k Q_{k,h}^g(s,\cdot), \pi_h^{k+1}(\cdot|s) - \pi_h^{k}(\cdot|s) \rangle + \frac{1}{\alpha}D(\pi_h^*(\cdot|s)|| \pi_h^1(\cdot|s))\\
        &\quad+ \frac{1}{\alpha} \sum_{k=1}^K \left(D(\pi_h^*(\cdot|s)|| \widetilde\pi_h^k(\cdot|s)) - D(\pi_h^*(\cdot|s)|| \pi_h^k(\cdot|s))\right)\\
        &\leq \sum_{k=1}^K\langle Q_{k,h}^r(s,\cdot) + Y_k Q_{k,h}^g(s,\cdot), \pi_h^{k+1}(\cdot|s) - \pi_h^{k}(\cdot|s) \rangle + \frac{\log|\calA|}{\alpha}\\
        &\quad+ \frac{1}{\alpha} \sum_{k=1}^K \left(D(\pi_h^*(\cdot|s)|| \widetilde\pi_h^k(\cdot|s)) - D(\pi_h^*(\cdot|s)|| \pi_h^k(\cdot|s))\right)
    \end{align*}
    where the second inequality follows from the nonnegativeness of KL divergence, and the last inequality is due to $\pi_h^1(\cdot|s) = 1/|\calA|$ for any $s$. Moreover,
    \begin{align}
        &\sum_{k=1}^K\langle Q_{k,h}^r(s,\cdot) + Y_k Q_{k,h}^g(s,\cdot), \pi_h^{k+1}(\cdot|s) - \pi_h^{k}(\cdot|s) \rangle \notag\\
        &\leq \sum_{k=1}^K \| Q_{k,h}^r(s,\cdot) + Y_k Q_{k,h}^g(s,\cdot)\|_\infty \|\pi_h^{k+1}(\cdot|s) - \pi_h^{k}(\cdot|s)\|_1 \notag\\
        &\leq \sum_{k=1}^K \| Q_{k,h}^r(s,\cdot) + Y_k Q_{k,h}^g(s,\cdot)\|_\infty \|\pi_h^{k+1}(\cdot|s) - \widetilde\pi_h^{k}(\cdot|s)\|_1 \notag\\
        &\quad+ \sum_{k=1}^K \| Q_{k,h}^r(s,\cdot) + Y_k Q_{k,h}^g(s,\cdot)\|_\infty \|\widetilde\pi_h^{k}(\cdot|s) - \pi_h^{k}(\cdot|s)\|_1 \notag\\
        &\leq \sum_{k=1}^K \alpha H^2(1+Y_k)^2 + \theta\sum_{k=1}^K \| Q_{k,h}^r(s,\cdot) + Y_k Q_{k,h}^g(s,\cdot)\|_\infty\|\piunif(\cdot|s) - \pi_h^k(\cdot|s)\|_1 \notag\\
        &=\bigO\left(\frac{\alpha H^6 K}{\gamma^2} + \frac{\theta K H^3}{\gamma}\right) \label{eq:omd 1}
    \end{align}
    where the first inequality is due to H\"older's inequality, the second inequality is due to the triangle inequality, the third inequality is due to \Cref{lem:softmax lipschitz} and that $\widetilde\pi_h^k(\cdot|s) = (1-\theta) \pi_h^k(\cdot|s) + \theta \piunif(\cdot|s)$, and the last equality is due to \Cref{lem:Yk bound}. Moreover,  
    \begin{align}\label{eq:omd 2}
    \begin{aligned}
        \frac{1}{\alpha} \sum_{k=1}^K \left(D(\pi_h^*(\cdot|s)|| \widetilde\pi_h^k(\cdot|s)) - D(\pi_h^*(\cdot|s)|| \pi_h^k(\cdot|s))\right) 
        &\leq \frac{1}{\alpha} \sum_{k=1}^K \theta \log|\calA| \leq \frac{\theta K \log|\calA|}{\alpha}.
    \end{aligned}
    \end{align}
    Finally, by applying \eqref{eq:omd 1} and \eqref{eq:omd 2}, we have
    \begin{align*}
        \sum_{k=1}^K\langle Q_{k,h}^r(s,\cdot) + Y_k Q_{k,h}^g(s,\cdot), \pi_h^{*}(\cdot|s) - \pi_h^{k}(\cdot|s) \rangle = \bigO\left(\frac{\alpha H^6 K}{\gamma^2} + \frac{\theta K }{\alpha} + \frac{1}{\alpha} + \frac{\theta K H^3}{\gamma}\right).
    \end{align*}
    Then it follows that 
    \begin{align*}
        \sum_{k=1}^K\bbE_{\pi^*}\left[\sum_{h=1}^H \langle Q_{k,h}^r(s_h,\cdot) + Y_kQ_{k,h}^g(s_h,\cdot), \pi_h^*(\cdot|s_h) - \pi_h^k(\cdot|s_h) \rangle | s_1\right] = \bigO\left(\frac{\alpha H^7 K}{\gamma^2} + \frac{\theta HK }{\alpha} + \frac{H}{\alpha}+\frac{\theta K H^4}{\gamma}\right).
    \end{align*}
    Finally, we have
    \begin{align*}
        \sum_{k=1}^K \left(V_1^{r^k,\pi^*}(s_1) + Y_k V_1^{g,\pi^*}(s_1) - V_{k,1}^r(s_1) - Y_k V_{k,1}^g(s_1) \right) =\bigO\left(\frac{\alpha H^7 K}{\gamma^2} + \frac{\theta HK }{\alpha} + \frac{H}{\alpha}+ \frac{\theta K H^4}{\gamma}\right).
    \end{align*}
    Since we set $\alpha = 1/(H^2\sqrt{K}),\ \theta = 1/K$, the proof is completed.
\end{proof}

\section{Proof of Theorem~\ref{thm:main}}
\begin{proof} 
By Lemmas \ref{lem: beta hat}, \ref{lem:ltv}, \ref{lem:mds term}, and \ref{lem:mds term for V - Q}, by union bound, the good event $\calE$ holds with probability at least $1-6\delta$. Moreover, under $\calE$ and $K\geq \max\{2H, H^2\}$, \Cref{lem:Yk bound} holds. Therefore, with probability at least $1-6\delta$, the good event $\calE$ and $Y_k = \bigO(H^2/\gamma)$ for all $k$. We assume these throughout the proof.

Next, we bound the regret. Note that the regret can be decomposed as follows.
\begin{align*}
    &\sum_{k=1}^K \left(V_1^{r^k,\pi^*}(s_1) - V_1^{r^k,\pi^k}(s_1)\right) \\
    &= \sum_{k=1}^K \left(V_1^{r^k,\pi^*}(s_1) + Y_k b - V_{k,1}^r(s_1) - Y_k V_{k,1}^g(s_1) \right) + \sum_{k=1}^K \left(V_{k,1}^r(s_1) - V_1^{r^k,\pi^k}(s_1)\right) + \sum_{k=1}^K Y_k\left(V_{k,1}^g(s_1)-b\right)\\
    &\leq \underbrace{\sum_{k=1}^K \left(V_1^{r^k,\pi^*}(s_1) + Y_k V_1^{g,\pi^*}(s_1) - V_{k,1}^r(s_1) - Y_k V_{k,1}^g(s_1) \right)}_{\textnormal{(I)}} \\
    &\quad+ \underbrace{\sum_{k=1}^K \left(V_{k,1}^r(s_1) - V_1^{r^k,\pi^k}(s_1)\right)}_{\textnormal{(II)}} + \underbrace{\sum_{k=1}^K Y_k\left(V_{k,1}^g(s_1)-b\right)}_{\textnormal{(III)}}
\end{align*}

    By \Cref{lem:omd}, 
    \begin{align*}
        \textnormal{(I)} =\bigO\left(\frac{H^5 \sqrt{K}}{\gamma^2} + H^3\sqrt{K} + \frac{H^4}{\gamma}\right).
    \end{align*}
    By \Cref{lem:bias}, 
    \begin{align*}
        \textnormal{(II)} = \bigO\left(\sqrt{dH^4 K} + \sqrt{d^2 H^3 K}\right).
    \end{align*}
    Note that (III) can be bounded as follows. Due to the dual update
    \begin{align*}
        0 
        &\leq  Y_{K+1}^2\\
        &= \sum_{k=1}^K (Y_{k+1}^2 - Y_k^2) \\
        &\leq \sum_{k=1}^K \left(2Y_k\eta(b- V_{k,1}^g(s_1) -\alpha H^3(1+Y_k) - 2\theta H^2) + \eta^2(b- V_{k,1}^g(s_1) -\alpha H^3(1+Y_k) - 2\theta H^2)^2\right).
    \end{align*}
    It can be rewritten as
    \begin{align*}
        \textnormal{(III)}&=\sum_{k=1}^K Y_k(V_{k,1}^g(s_1) -b) \\
        &\leq -\sum_{k=1}^K Y_k(\alpha H^3(1+Y_k) + 2\theta H^2) + \sum_{k=1}^K \frac{\eta}{2} \left(b-V_{k,1}^g(s_1) - \alpha H^3 (1+Y_k) -2\theta H^2\right)^2\\
        &\leq \sum_{k=1}^K \frac{3\eta}{2} \left(H^2 + \alpha^2 H^6 (1+Y_k)^2 + 4\theta^2 H^4\right)\\
        &=\bigO\left(H\sqrt{K} + \frac{H^5}{\gamma^2 \sqrt{K}}\right)
    \end{align*}
    where the second inequality follows from the Cauchy-Schwarz inequality, and the equality follows from the choice of $\eta, \alpha, \theta$, the assumption $K\geq H^2$, and \Cref{lem:Yk bound}.
    Finally, we deduce that 
    \begin{align*}
        \Regret = \bigO\left(\sqrt{dH^4 K} + \sqrt{d^2 H^3 K} + \frac{H^5 \sqrt{K}}{\gamma^2} + H^3\sqrt{K} + \frac{H^4}{\gamma}\right)
    \end{align*}
    Next, we analyze constraint violation. Note that the violation can be decomposed as
    \begin{align*}
    \sum_{k=1}^K (b - V_1^{g,\pi^k}(s_1)) = \underbrace{\sum_{k=1}^K (b - V_{k,1}^{g}(s_1))}_{\textnormal{(IV)}} + \underbrace{\sum_{k=1}^K (V_{k,1}^{g}(s_1) - V_1^{g,\pi^k}(s_1))}_{\textnormal{(V)}}
\end{align*}
To bound (IV), note that $Y_{k+1}\geq (1-\alpha\eta H^3) Y_k  + \eta(b- V_{k,1}^g(s_1) - \alpha H^3 - 2\theta H^2)$. This leads to
\begin{align*}
    b - V_{k,1}^g(s_1) \leq \frac{Y_{k+1} - Y_k}{\eta} + \alpha  H^3 Y_k + \alpha H^3 + 2\theta H^2.
\end{align*}
By summing the above inequality over $k=1,\ldots,K$,
\begin{align*}
    \sum_{k=1}^K (b- V_{k,1}^g(s_1)) 
    &\leq \frac{Y_{K+1}}{\eta} + \alpha H^3 \sum_{k=1}^K Y_k + \alpha KH^3 + 2\theta KH^2 \\
    &= \bigO\left(\frac{H^3}{\gamma}\sqrt{K}\right).
\end{align*}
By \Cref{lem:bias}, 
\begin{align*}
    \textnormal{(V)} = \bigO\left(\sqrt{dH^4 K} + \sqrt{d^2 H^3 K}\right).
\end{align*}
Finally, 
\begin{align*}
    \Violation = \bigO\left(\sqrt{dH^4 K} + \sqrt{d^2 H^3 K}+ \frac{H^3\sqrt{K}}{\gamma}\right).
\end{align*}
\end{proof}

\section{Details of Numerical Experiments}

In this section, we provide additional details of the numerical experiments in \Cref{sec:numerical}. Our experimental environment is adapted from the linear mixture MDP instance of \cite{he2022near}, with modifications to incorporate a constraint and adversarially changing rewards. The implementation and scripts used to reproduce the experiments are publicly available at \url{https://github.com/kihyun-yu/pd-powers}.

We set the horizon to $H=10$ and the feature dimension to $d=5$. The state and action spaces are $\calS=\{0,\ldots,H+1\}$ and $\calA=\{-1,1\}^{d-1}$, respectively, and the initial state is $s_1=0$. We set the constraint threshold to $b=6$. For $s<H$, the transition probabilities are given by $P_h(s+1\mid s,a)=0.95-0.01\bm 1^\top a$ and $P_h(H+1\mid s,a)=1-P_h(s+1\mid s,a)$. The states $H$ and $H+1$ are absorbing, so that $P_h(s\mid s,a)=1$ for $s\in\{H,H+1\}$.

To model adversarially changing rewards, the reward function changes periodically across episodes. For $s<H$, we define $r_h^k(s,a)=0.4\sum_{i=1}^{d-1}(a_i+1)/(2(d-1))$ when $\lfloor k/10\rfloor\equiv0\pmod 2$, and $r_h^k(s,a)=0.4\bigl(1-\sum_{i=1}^{d-1}(a_i+1)/(2(d-1))\bigr)$ otherwise. For the absorbing states, we set $r_h^k(H,a)=0$ and $r_h^k(H+1,a)=1$ for every $k\in[K]$. The fixed constraint function is defined as $g_h(s,a)=\sum_{i=1}^{d-1}(a_i+1)/(2(d-1))$ for $s<H$, and $g_h(s,a)=0$ for $s\in\{H,H+1\}$.



In addition to the main comparison, we conduct sensitivity experiments to investigate how the empirical performance of PD-POWERS changes with the underlying problem parameters. We vary one parameter at a time, including the state-space size $|\mathcal S|$, the action-space size $|\mathcal A|$, and the horizon $H$, while keeping the remaining parameters fixed at their default values. The specific configurations considered in each experiment are indicated in the corresponding figure legends. We use the same reward-generation mechanism and constraint construction as in the main experiment.

Figures~\ref{fig:sensitivity_regret} and~\ref{fig:sensitivity_violation} report the cumulative regret and constraint violation, respectively, under these different configurations. Overall, the results show that PD-POWERS maintains stable learning behavior across the considered problem settings, while the magnitude of regret and constraint violation changes as the difficulty of the underlying CMDP varies. These experiments complement the baseline comparison in \Cref{fig:main} and demonstrate that the observed behavior of PD-POWERS is not specific to a single choice of problem parameters.

\begin{figure*}[t]
\centering
\begin{subfigure}[t]{0.49\textwidth}
\centering
\includegraphics[width=0.8\linewidth]{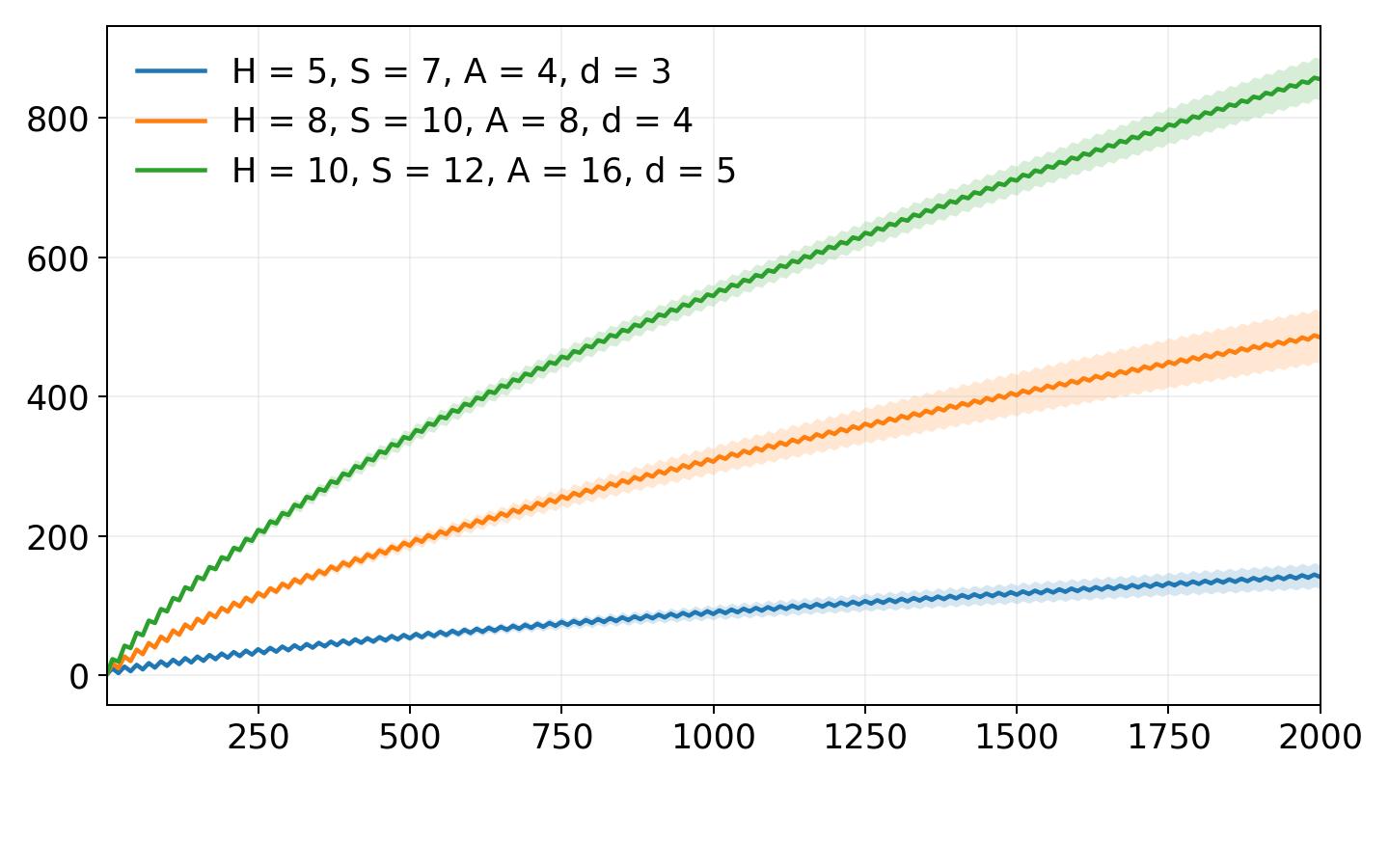}
\caption{Regret}
\label{fig:sensitivity_regret}
\end{subfigure}
\hfill
\begin{subfigure}[t]{0.49\textwidth}
\centering
\includegraphics[width=0.8\linewidth]{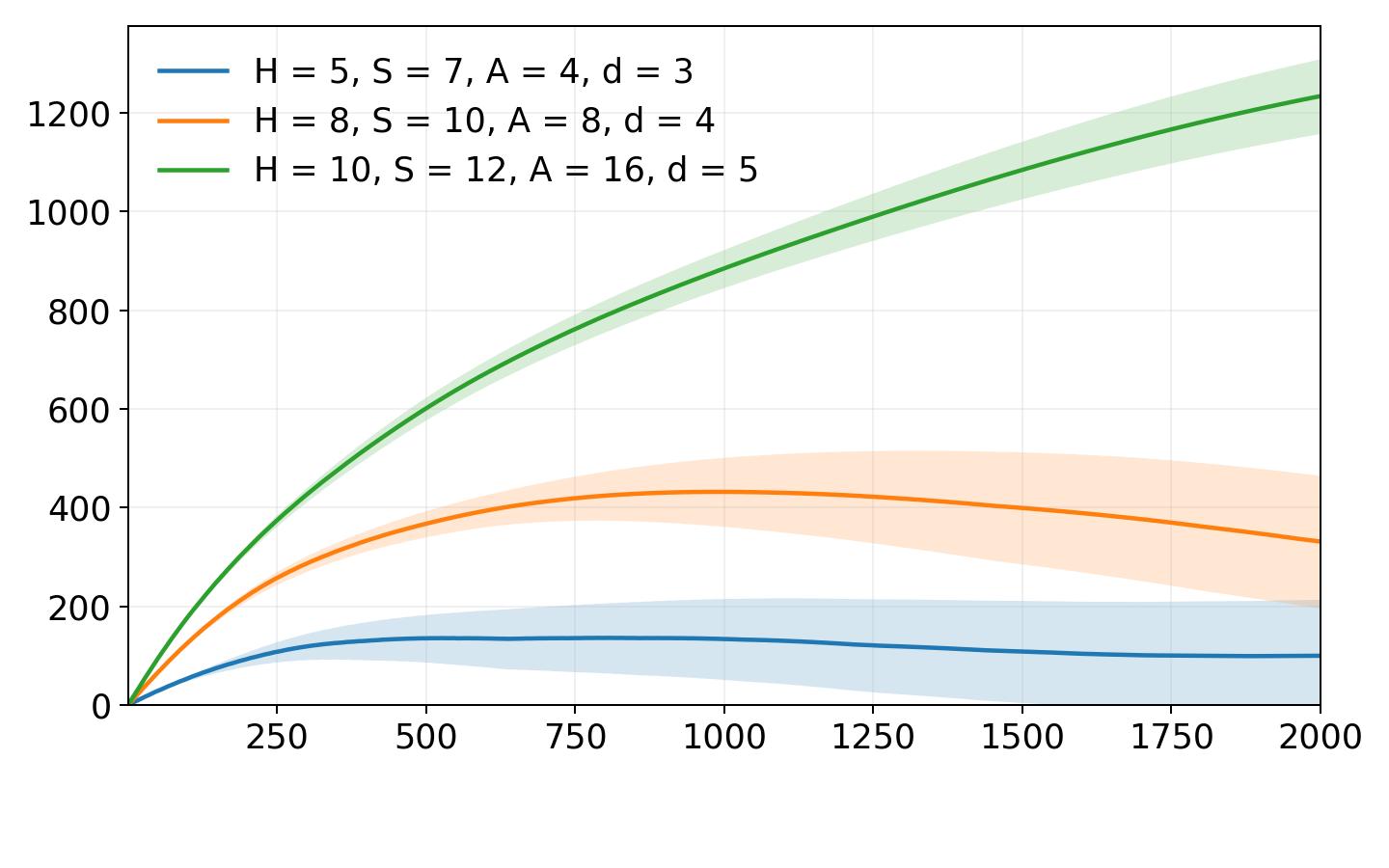}
\caption{Constraint Violation}
\label{fig:sensitivity_violation}
\end{subfigure}
\caption{Sensitivity of PD-POWERS to different problem parameters. In each experiment, one parameter is varied while the remaining parameters are fixed at their default values.}
\label{fig:sensitivity}
\end{figure*}


\section{Auxiliary Lemmas}
\begin{lemma}[Lemma 11 in \cite{abbasi2011improved}]\label{lem:elliptical}
    For any $\lambda > 0$ and sequence $\{\bfx_t\}_{t=1}^T \subset \bbR^d$ for $t \in \{0,\ldots,T\}$, define $Z_t = \lambda I + \sum_{i=1}^t \bfx_i\bfx_i^\top$. Then, provided that $\|\bfx_t\|_2 \leq L$ holds for all $t\in [T],$ we have
    \begin{align*}
        \sum_{t=1}^T \min \{1, \|\bfx_t\|_{Z_{t-1}^{-1}}^2\} \leq 2d\log\frac{d\lambda + TL^2}{d\lambda}.
    \end{align*}
\end{lemma}

\begin{lemma}[Theorem 4.1 in \cite{zhou2021nearly}]\label{lem:zhou thm 4.1}
    Let $\{\calG_t\}_{t=1}^\infty$ be a filtration, $\{\bfx_t, \eta_t\}_{t\ge1}$ a stochastic process so that $\bfx_t\in\bbR^d$ is $\calG_t$-measurable and $\eta_t \in \bbR$ is $\calG_{t+1}$-measurable. Fix $R,L,\sigma,\lambda > 0$, $\bm{\mu}^* \in \bbR^d$. For $t\geq 1$, let $y_t = \langle \bm{\mu}^*,\bfx_t \rangle + \eta_t$ and suppose that $\eta_t, \bfx_t$ also satisfy
    \begin{align*}
        |\eta_t| \leq R, \ \bbE[\eta_t |\calG_t] = 0, \ \bbE[\eta_t^2 | \calG_t] \leq \sigma^2, \ \|\bfx_t\|_2 \leq L.
    \end{align*}
    Then, for any $0<\delta<1$, with probability at least $1-\delta$, we have
    \begin{align*}
        \forall t > 0, \ \left\|\sum_{i=1}^t \bfx_i \eta_i\right\|_{Z_t^{-1}} \leq \beta_t, \ \|\bm{\mu}_t - \bm{\mu}^*\|_{Z_t}\leq \beta_t + \sqrt{\lambda}\|\bm{\mu}^*\|_2,
    \end{align*}
    where for $t\geq 1$, $\bm{\mu}_t = Z_t^{-1}\bfb_t$, $Z_t = \lambda I + \sum_{i=1}^t \bfx_i\bfx_i^\top$, $\bfb_t = \sum_{i=1}^t y_i \bfx_i$ and
    \begin{align*}
        \beta_t = 8\sigma\sqrt{d\log(1 + tL^2/(d\lambda))\log(4t^2/\delta)} + 4R\log(4t^2 / \delta).
    \end{align*}
\end{lemma}

\begin{lemma}[Lemma 31 of \cite{wei2020online}]\label{lem:KL mixing}
Let $\pi_1, \pi_2$ be two probability distributions in $\Delta(\mathcal{A})$.  
Let $\tilde{\pi}_2 = (1-\theta)\pi_2 + \theta/|\mathcal{A}|$ where $\theta \in (0,1)$.  
Then,
\begin{align*}
    D(\pi_1 \| \tilde{\pi}_2) - D(\pi_1 \| \pi_2) \leq \theta \log |\mathcal{A}|,\quad D(\pi_1 \| \tilde{\pi}_2) \leq \log(|\mathcal{A}|/\theta).
\end{align*}
\end{lemma}

\begin{lemma}[Lemma 22 of \cite{yu2026primaldual}]\label{lem:policy inequality}
    Let $\pi_h^k:\calS \to \Delta(\calA)$ be any policies. For $\theta \in [0,1]$, let $\widetilde\pi_h^k(\cdot|s)=(1-\theta)\pi_h^k(\cdot|s) + \theta\piunif(\cdot|s)$. For $ Q_{k,h}^r,  Q_{k,h}^g:\calS \times\calA \to [0,H]$, $Y_k \in \bbR_+$, and $\alpha >0$, let $\pi^{k+1}(\cdot\mid s) \propto \widetilde\pi^k(\cdot\mid s)\exp(\alpha ( Q_{k,h}^r(s,\cdot) + Y_k  Q_{k,h}^g(s,\cdot))$. For any $s\in\calS$, we have
    \begin{enumerate}
        \item $\|\pi_h^{k+1}(\cdot\mid s) - \widetilde\pi_h^k(\cdot\mid s)\|_1 \leq \alpha H(1+Y_k),$
        \item $\left|\langle \pi_h^{k+1}(\cdot|s) - \pi_h^{k}(\cdot|s),  Q_{k,h}^g(s,\cdot) \rangle\right| \leq \alpha H^2(1+Y_k) + 2\theta H.$
        \item $\langle \pi_h^{k+1}(\cdot|s) - \pi_h^{k}(\cdot|s),  Q_{k,h}^r(s,\cdot) \rangle - \frac{1}{\alpha} D(\pi_h^{k+1}(\cdot|s)|| \widetilde\pi_h^{k}(\cdot|s)) \leq \alpha H^2/2 + 2H\theta.$
    \end{enumerate}
\end{lemma}

\begin{lemma}[Lemma 1 of \cite{shani2020optimistic}]\label{lem:extended value diff}
Let $\pi, \pi'$ be two policies, and let $\mathcal{M} = (H, \mathcal{S}, \mathcal{A}, P, r,s_1)$ be an MDP. 
For all $h\in [H]$, let ${Q}_h^r:\calS\times\calA \rightarrow \bbR$ be an arbitrary function, and let
${V}_h^r(s) = \left\langle {Q}_h^r(s, \cdot), \pi'_h(\cdot \mid s) \right\rangle$ for all $s\in\calS$.
Then,
\begin{align*}
V_1^{r,\pi}(s_1) - V_1^r(s_1)
&=  \bbE_{\pi} \left[ \sum_{h=1}^H
\left\langle  {Q}_h^r(s_h, \cdot), \pi_h(\cdot \mid s_h) - \pi'_h(\cdot \mid s_h) \right\rangle \,\middle|\, s_1\right] + \bbE_{\pi} \left[ \sum_{h=1}^H
r_h(s_h, a_h) + P_h{V}_{h+1}^r(s_h, a_h) - {Q}_h^r(s_h, a_h) 
\,\middle|\, s_1\right],
\end{align*}
where $V_1^{r,\pi} = \bbE_{\pi}[\sum_{j=1}^H r_h(s_h,a_h)| s_1]$.
\end{lemma}

\begin{lemma}[Lemma 1 of \cite{wei2020online}]\label{lem:pushback}
    Let $\Delta,\ \interior(\Delta)$ be the probability simplex and its interior, respectively, and let $f:\calC \rightarrow \bbR$ be a convex function. Fix $\alpha >0,\ y\in \interior(\Delta)$. Suppose $x^* \in \argmax_{x\in \Delta} f(x) - (1/\alpha) D(x||y)$ and $x^*\in \interior(\Delta)$, then, for any $z\in \Delta$,
    \[
        f(x^*) - \frac{1}{\alpha} D(x^* ||y) \geq f(z) - \frac{1}{\alpha} D(z||y) + \frac{1}{\alpha} D(z||x^*).
    \]
\end{lemma}

\begin{lemma}[Lemma 33 of \cite{kitamura2025provably}]\label{lem:softmax lipschitz}
Let $Q_1, Q_2:\calA \rightarrow \bbR$ be two functions. For $\alpha > 0$, let $\pi_1 \propto \exp(\alpha Q_1),\ \pi_2 \propto \exp(\alpha Q_2)$. Then we have
\[
    \|\pi_1 - \pi_2\|_1 \leq 8\alpha\|Q_1 - Q_2\|_\infty.
\]
\end{lemma}